\documentclass{article}

\usepackage[margin=1.5in]{geometry}
\usepackage{fullpage}

\usepackage{times}
\usepackage[round]{natbib}

\usepackage[utf8]{inputenc}
\usepackage{mathtools}
\usepackage{amsmath}
\usepackage{amsfonts}
\usepackage{amssymb}
\usepackage{amsthm}
\usepackage{booktabs}
\usepackage{color}
\usepackage[mathscr]{euscript}
\usepackage{dsfont}
\usepackage{hyperref}
\hypersetup{
    colorlinks,
    linkcolor={red!50!black},
    citecolor={blue!50!black},
    urlcolor={blue!80!black}
}

\usepackage[utf8]{inputenc}
\usepackage[T1]{fontenc}
\usepackage{url}

\usepackage{nicefrac}
\usepackage{mathtools}
\usepackage{amsmath}
\usepackage{amsfonts}
\usepackage{amssymb}
\usepackage{bbm}
\usepackage{mathrsfs}
\usepackage{upgreek}
\usepackage{booktabs}
\usepackage{xcolor}
\usepackage[mathscr]{euscript}
\usepackage{dsfont}
\usepackage{stmaryrd}
\usepackage{hyperref}
\hypersetup{
    colorlinks,
    linkcolor={red!80!black},
    citecolor={blue!60!black},
    urlcolor={blue!80!black}
}
\usepackage{graphicx}

\usepackage{natbib}
\usepackage{float}
\usepackage[justification=centering]{caption}
\usepackage{tikz,pgfplots}
\pgfplotsset{
        table/search path={contextual/},
    }

\usepackage{subcaption}
\usepackage{wrapfig}

\usepackage{ifthen}
\usepackage{xargs}
\usepackage{algorithm}
\usepackage{algorithmic}

\usepackage{comment}
\allowdisplaybreaks

\usepackage{lmodern} %% improve pdf rendering

\usepackage{aliascnt}
\usepackage{cleveref}
\usepackage{autonum}

\makeatletter
\newtheorem{theorem}{Theorem}
\crefname{theorem}{theorem}{Theorems}
\Crefname{Theorem}{Theorem}{Theorems}

\newtheorem*{lemma_nonumber*}{Lemma}

\newtheorem{lemma}{Lemma}
\crefname{lemma}{lemma}{lemmas}
\Crefname{Lemma}{Lemma}{Lemmas}

\newtheorem{corollary}{Corollary}
\crefname{corollary}{corollary}{corollaries}
\Crefname{Corollary}{Corollary}{Corollaries}

\newtheorem{proposition}{Proposition}
\crefname{proposition}{proposition}{propositions}
\Crefname{Proposition}{Proposition}{Propositions}

\newtheorem{definition}{Definition}
\crefname{definition}{definition}{definitions}
\Crefname{Definition}{Definition}{Definitions}

\crefname{remark}{remark}{remarks}
\Crefname{Remark}{Remark}{Remarks}

\crefname{example}{example}{examples}
\Crefname{Example}{Example}{Examples}

\crefname{figure}{figure}{figures}
\Crefname{Figure}{Figure}{Figures}

\newtheorem{assumption}{\textbf{A}\hspace{-3pt}}
\crefformat{assumption}{{\textbf{A}}#2#1#3}

\crefformat{assumptionF}{{\textbf{F}}#2#1#3}

\Crefname{assumptionB}{\textbf{B}\hspace{-3pt}}{\textbf{B}\hspace{-3pt}}
\crefname{assumptionB}{\textbf{B}}{\textbf{B}}

\Crefname{assumptionC}{\textbf{C}\hspace{-3pt}}{\textbf{C}\hspace{-3pt}}
\crefname{assumptionC}{\textbf{C}}{\textbf{C}}

\Crefname{assumptionH}{\textbf{H}\hspace{-3pt}}{\textbf{H}\hspace{-3pt}}
\crefname{assumptionH}{\textbf{H}}{\textbf{H}}

\Crefname{assumptionT}{\textbf{T}\hspace{-3pt}}{\textbf{T}\hspace{-3pt}}
\crefname{assumptionT}{\textbf{T}}{\textbf{T}}

\Crefname{assumptionT}{\textbf{T}\hspace{-3pt}}{\textbf{T}\hspace{-3pt}}
\crefname{assumptionT}{\textbf{T}}{\textbf{T}}

\Crefname{assumptionL}{\textbf{L}\hspace{-3pt}}{\textbf{L}\hspace{-3pt}}
\crefname{assumptionL}{\textbf{L}}{\textbf{L}}

\Crefname{assumptionQ}{\textbf{Q}\hspace{-3pt}}{\textbf{Q}\hspace{-3pt}}
\crefname{assumptionQ}{\textbf{Q}}{\textbf{Q}}

\Crefname{assumptionAR}{\textbf{AR}\hspace{-3pt}}{\textbf{AR}\hspace{-3pt}}
\crefname{assumptionAR}{\textbf{AR}}{\textbf{AR}}

\newcommand{\bE}{\mathbb{E}}
\newcommand{\bN}{\mathbb{N}}
\newcommand{\bP}{\mathbb{P}}
\newcommand{\bR}{\mathbb{R}}
\newcommand{\bX}{\mathbb{X}}

\newcommand{\sL}{\mathscr{L}}

\newcommand{\eps}{\varepsilon}
\newcommand{\sig}{\sigma^2}

\newcommand{\Y}[1]{Y^{(#1)}}

\newcommand{\ee}[1]{\eps^{(#1)}}

\newcommand{\argmax}{\operatorname*{arg\,max}}
\newcommand{\argmin}{\operatorname*{arg\,min}}

\newcommand{\dist}{\mbox{dist}}

\newcommand{\Tr}{\mbox{Tr}}
\newcommand{\Sp}{\mbox{Sp}}
\newcommand{\Gram}{\mbox{Gram}}

\newcommand{\et}{\quad\mbox{and}\quad}

\newcommand{\bigo}{\mathcal{O}}
\newcommand{\bigot}{\widetilde{\mathcal{O}}}

\newcommand{\pa}[1]{\left(#1\right)}
\newcommand{\II}[1]{\mathbb{I}{\left\{#1\right\}}}
\newcommand{\dpart}[2]{\dfrac{\partial #1}{\partial #2}}

\DeclarePairedDelimiter\abs{\lvert}{\rvert}

\newcommand{\ceil}[1]{\left\lceil #1 \right\rceil}

\newcommand{\iid}{i.i.d}

\newcommand{\la}{\langle}
\newcommand{\ra}{\rangle}

\newcommand{\norm}[1]{\left\Vert#1\right\Vert}
\newcommand{\normi}[1]{\norm{#1}_{\infty}}

\newcommand{\normd}[1]{\norm{#1}_2}
\newcommand{\normpsi}[2]{\norm{#1}_{\psi_{#2}}}

\renewcommand{\t}{^{\top}}
\newcommand{\pinv}{^{-1}}
\newcommand{\st}{^{\star}}
\newcommand{\bst}{\beta\st}
\newcommand{\pst}{p\st}
\newcommand{\bht}{\hat{\beta}}

\newcommand{\Dk}{\Delta^K}
\newcommand{\Dd}{\Delta^d}

\newcommand{\op}{\Omega(p)}
\newcommand{\ops}{\Omega(p\st)}
\newcommand{\hop}{\hat{\Omega}(p)}
\newcommand{\hsig}{\hat{\sigma}}

\newcommand{\osig}{\overline{\sigma}}
\newcommand{\hp}{\hat{p}}
\newcommand{\mud}{\mu^{(2)}}
\newcommand{\hmud}{\hat{\mu}^{(2)}}
\newcommand{\hmu}{\hat{\mu}}
\newcommand{\inv}{^{-1}}
\newcommand{\XX}{X_k X_k^{\top}}
\newcommand{\skk}{\sum_{k=1}^K}
\newcommand{\skd}{\sum_{k=1}^d}
\newcommand{\sinn}{\sum_{i=1}^n}
\newcommand{\lm}{\lambda_{\min}}
\DeclareMathOperator{\Var}{Var}
\DeclareMathOperator{\Cof}{Cof}
\DeclareMathOperator{\Com}{Com}

\newcommand{\fraca}[2]{\left. #1 / #2 \right.}
\newcommand{\transpose}{^{\top}}

\newcommand{\defEns}[1]{\left\lbrace #1 \right\rbrace }
\newcommand{\ooint}[1]{(#1)}

\newcommand{\parenthese}[1]{\left(#1\right)}

\def\eqsp{\;}

\newcommand{\ie}{\textit{i.e.}}
\newcommand{\eg}{\textit{e.g.}}

\allowdisplaybreaks

\newcommand{\mail}[1]{Email: \href{mailto:#1}{\textcolor{black}{#1}}}

\begin{document}

\newgeometry{margin=1.5in}

\title{\Large{\bf{Online A-Optimal Design and Active Linear Regression}}}

\author{Xavier Fontaine\thanks{ENS Paris-Saclay --- \mail{fontaine@cmla.ens-cachan.fr}}
\and
Pierre Perrault\thanks{ENS Paris-Saclay \& Inria --- \mail{pierre.perrault@inria.fr}}
\and
Michal Valko\thanks{Deepmind --- \mail{valkom@deepmind.com}}
\and
Vianney Perchet\thanks{ENSAE \& Criteo AI Lab --- \mail{vianney.perchet@normalesup.org}}
}

\date{}

\maketitle

\begin{abstract}
We consider in this paper the problem of optimal experiment design where a decision maker can choose which points to sample to obtain an estimate $\hat{\beta}$ of the hidden parameter $\beta^{\star}$ of an underlying linear model.
The key challenge of this work lies in the heteroscedasticity assumption that we make, meaning that each covariate has a different and unknown variance.
The goal of the decision maker is then to figure out on the fly the optimal way to allocate the total budget of $T$ samples between covariates, as sampling several times a specific one  will reduce the variance of the estimated model around it (but at the cost of a possible higher variance elsewhere).
By trying to minimize the $\ell^2$-loss $\mathbb{E} [\lVert\hat{\beta}-\beta^{\star}\rVert^2]$ the decision maker is actually minimizing the trace of the covariance matrix of the problem, which corresponds then to online A-optimal design.
Combining techniques from  bandit and convex optimization we propose a new active sampling algorithm and we compare it with existing ones. We provide theoretical guarantees of this algorithm in different settings, including a $\mathcal{O}(T^{-2})$ regret bound in the case where the covariates form a basis of the feature space, generalizing and improving existing results. Numerical experiments validate our theoretical findings.
\end{abstract}

\section{Introduction and related work}

A classical problem in statistics consists in estimating an unknown quantity, for example the mean of a random variable, parameters of a model, poll results or the efficiency of a medical treatment.
In order to do that, statisticians usually build estimators which are random variables based on the data, supposed to approximate the quantity to estimate.
A way to construct an estimator is to make experiments and to gather data on the estimand. In the polling context an experiment consists for example in interviewing people in order to know their voting intentions. However if one wants to obtain a ``good'' estimator, typically an unbiased estimator with low variance, the choice of which experiment to run has to be done carefully. Interviewing similar people might indeed lead to a poor prediction.
In this work we are interested in the problem of optimal design of experiments, which consists in choosing adequately the experiments to run in order to obtain an estimator with small variance.
We focus here on the case of heteroscedastic linear models with the goal of actively constructing the design matrix.
Linear models, though possibly sometimes too simple, have been indeed widely studied and used in practice due to their interpretability and can be a first good approximation model for a complex problem.

The original motivation of this problem comes from use cases where obtaining the label of a sample is costly, hence choosing carefully  which points to sample in a regression task is crucial.
Consider for example the problem of controlling the wear of manufacturing machines in a factory~\citep{antos2010active}, which requires a long and manual process.
The wear can be modeled as a linear function of some features of the machine (age, number of times it has been used, average temperature, ...) so that two machines with the same parameters will have similar wears.
Since the inspection process is manual and complicated,  results are noisy and this noise  depends on the machine: a new machine, slightly worn,  will often be in a good state, while the state of heavily worn machines can vary a lot. Thus evaluating the linear model for the wear requires additional
examinations of some machines and less inspection of others.
Another motivating example comes from econometrics, typically in income forecasting. It is usually  assumed that the annual income is influenced by the individual’s education level, age, gender, occupation, \textit{etc.} through a linear model. Polling is also an issue in this context: what kind of individual to poll to gain as much information as possible about an explanatory variable?

The field of optimal experiment design~\citep{pukelsheim2006optimal} aims precisely at choosing which experiment to perform in order to minimize an objective function within a budget constraint.
In experiment design, the distance of the produced hypothesis to the true one is measured by the covariance matrix of the error~\citep{boyd}. There are several criteria that can be used to minimize a covariance matrix, the most popular being A, D and E-optimality.
In this paper we focus on A-optimal design whose goal is to minimize the trace of the covariance matrix.
Contrary to several existing works which solve the A-optimal design problem in an offline manner in the homoscedastic setting~\citep{sagnol2010optimal,yang2013optimal,optdesignoffline} we are interested here in proposing an algorithm which solves this problem sequentially, with the additional challenge that each experiment has an unknown and different variance.

Our problem is therefore close to ``active learning'' which is more and more popular nowadays  because of the exponential growth of datasets and the cost of labeling data.
Indeed, the latter may be tedious and require expert knowledge, as in the domain of medical imaging. It is therefore essential to choose wisely which data to collect and to label, based on the information gathered so far.
Usually, machine learning agents are assumed to be passive in the sense that the data is seen as a fixed and given input that cannot be modified or optimized.
However, in many cases, the agent can be able to appropriately select the data \citep{DesignBook}. Active learning specifically studies the optimal ways to perform data selection~\citep{jordan} and this is crucial as one of the current limiting factors of machine learning algorithms are computing costs, that can be reduced since all examples in a dataset do not have equal importance~\citep{selective_sampling}.
This approach has many practical applications: in online marketing where one wants to estimate the potential impact of new products on customers, or in online polling where the different options do not have the same variance~\citep{PoolingBook}.

In this paper we consider therefore a decision maker who has a limited experimental budget and aims at learning some latent linear model.
The goal is to build a predictor $\bht$ that estimates the unknown parameter of the linear model $\bst$, and that minimizes $\bE[\|\bht-\bst\|^2]$. The key point here is that the design matrix is constructed sequentially and actively by the agent:  at each time step, the decision maker chooses  a ``covariate'' $X_k \in \bR^d$ and receives a noisy output $X_k\t\bst + \varepsilon$.
The quality of the predictor is measured through its variance.
The agent will repeatedly query the different available covariates in order to obtain more precise estimates of their values.
Instinctively a covariate with small variance should not be sampled too often since its value is already quite precise. On the other hand, a noisy covariate will be sampled more often.
The major issue lies in the heteroscedastic assumption: the unknown variances must be learned to wisely sample the points.

\citet{antos2010active} introduced a specific variant of our setting where the environment providing the data is assumed to be stochastic and i.i.d. across rounds.
More precisely, they studied this problem using the framework of stochastic multi-armed bandits (MAB) by considering a set of $K$ probability distributions (or arms), associated with $K$ variances.
Their objective is to define an allocation strategy over the arms to estimate their expected values uniformly well.
Later, the analysis and results have been improved by \citet{carpentier}. However, this line of work is actually focusing on the case where the covariates are only vectors of the canonical basis of $\bR^d$, which gives a simpler closed form linear regression problem.

There have been some recent works on MAB with heteroscedastic noise~\citep{cowan2015normal, kirschner2018information} with natural connections to this paper. Indeed,  covariates could somehow be interpreted as  contexts in contextual bandits. The most related setting might be the one of~\citet{soare}. However, they are mostly concerned about best-arm identification while recovering the latent parameter $\bst$ of the linear model is a more challenging task (as each decision has an impact on the loss). In that sense we improve the results of~\citet{soare} by proving a bound on the regret of our algorithm.
Other works as~\citep{pricechen} propose active learning algorithms aiming at finding a constant factor approximation of the classification loss while we are focusing on the statistical problem of recovering $\bst$.
Yet another similar setting has been introduced in \citep{traceucb}. In this setting the agent has to estimate several linear models in parallel and for each covariate (that appears randomly), the agent has to decide which model to estimate.
Other works studied the problem of active linear regression, and for example~\citet{modelselection} proposed an algorithm conducting active learning and model selection simultaneously but without any theoretical guarantees.
More recently \citet{threshold} have studied the setting of active linear regression with thresholding techniques in the homoscedastic case.
An active line of research has also been conducted in the domain of random design linear regression~\citep{hsu,stratification,derezinski}. In these works the authors aim at controlling the mean-squared regression error $\mathbb{E}[(X\t \beta-Y)^2]$ with a minimum number of random samples $X_k$.
Except from the loss function that they considered, these works differ from ours in several points: they generally do not consider the heteroscedastic case and their goal is to minimize the number of samples to use to reach an $\varepsilon$-estimator while in our setting the total number of covariates $K$ is fixed. \citet{allen2020} provide a similar analysis but under the scope of optimal experiment design.
Another setting similar to ours is introduced in~\citep{hardmargin}, where active linear regression with a hard-margin criterion is studied. However,  the  minimization of  the classical $\ell^2$-norm of the difference between the true parameter of the linear model and its estimator seems to be a more natural criterion, which justifies our investigations.

In this work we adopt a different point of view from the aforementioned existing works.
We consider A-optimal design under the heteroscedasticity assumption and we generalize MAB results to the non-coordinate basis setting with two different algorithms taking inspiration from the convex optimization and bandit literature.
We prove optimal $\bigot(T^{-2})$ regret bounds for $d$ covariates and provide a weaker guarantee for more than $d$ covariates.
Our work emphasizes the connection between MAB and optimal design, closing open questions in A-optimal design.
Finally we corroborate our theoretical findings with numerical experiments.

\section{Setting and description of the problem}

\subsection{Motivations and description of the setting}

Let $X_1, \dots, X_K \in \bR^d$ be $K$ covariates available to some agent who can successively sample each of them (several times if needed). Observations $Y$ are generated by a standard linear model, \ie,
\[
Y=X\t \bst + \varepsilon \quad \mbox{with }\bst \in \bR^d \eqsp .
\]
Each of these covariates correspond to an experiment that can be run by the decision maker to gain information about the unknown vector $\bst$. The goal of optimal experiment design is to choose the experiments to perform from a pool of possible design points $\defEns{X_1, \dots, X_K}$ in order to obtain the best estimate $\bht$ of $\bst$ within a fixed budget of $T \in \bN^*$ samples.
In classical experiment design problems the variances of the different experiments are supposed to be equal. Here we consider the more challenging setting where each covariate has a specific and \textit{unknown} variance $\sigma_k^2$, \ie, we suppose that when $X_k$ is queried for the $i$-th time the decision maker observes
\[
Y^{(i)}_{k}=X_k\t \bst + \varepsilon_{k}^{(i)} \eqsp,
\]
where $\bE[\eps_{k}^{(i)}]=0$, $\Var[\eps_{k}^{(i)}]=\sig_k>0$ and $\eps_{k}^{(i)}$ is $\kappa^2$-subgaussian. We assume also that the $\varepsilon_{k}^{(i)}$ are independent from each other.
This setting corresponds actually to online optimal experiment design since the decision maker has to design sequentially the sampling policy, in an adaptive manner.

A naive sampling strategy is to equally sample each covariate $X_k$.
In our heteroscedastic setting, this will not produce the most precise estimate of $\bst$ because of the different variances $\sigma^2_k$.
Intuitively a point $X_k$ with a low variance will provide very precise information on the value $X_k\t\bst$ while a point with a high variance will not give much information (up to the converse effect of the norm $\|X_k\|$).
This indicates that a point with high variance should be sampled more often than a point with low variance. Since the variances $\sigma_k^2$ are unknown, we need at the same time to estimate $\sigma^2_k$ (which might require lots of samples of $X_k$ to be precise) and to minimize the estimation error (which might require only a few examples of some covariate $X_k$). There is then a tradeoff between gathering information on the values of $\sigma^2_k$ and using it to optimize the  loss; the fact that this loss is global, and not cumulative, makes this tradeoff ``exploration \textit{vs.}\ exploitation'' much more intricate than in standard multi-armed bandits.

Usual algorithms handling global losses are rather slow \citep{Agrawal_knapsacks,Mannor} or dedicated to specific well-posed problems with closed form losses~\citep{antos2010active,carpentier}.
Our setting can be seen as an extension of the two aforementioned works who aim at estimating the means of a set of $K$ distributions. Noting $\mu=(\mu_1, \dots, \mu_K)^{\top}$ the vector of the means of those distributions and $X_i=e_i$ the $i^{\textrm{th}}$ vector of the canonical basis of $\bR^K$, we see (since $X_i^{\top}\mu=\mu_i$) that their objective is actually to estimate the parameter $\mu$ of a linear model.
This setting is a particular case of ours since the vectors $X_i$ form the canonical basis of $\mathbb{R}^K$.

\subsection{Definition of the loss function}

As we mentioned it before, the decision maker can be led to sample several times the same design point $X_k$ in order to obtain a more precise estimate of its response $X_k\t\bst$. We denote therefore by $T_k \geq 0$ the number of samples of $X_k$, hence $T = \sum_{k=1}^K T_k$. For each $k \in [K]$\footnote{$[K]=\defEns{1,\dots,K}$.}, the linear model yields the following
\[
T_k\pinv \sum_{i=1}^{T_k} \Y{i}_k=X_k^T\bst + T_k\pinv\sum_{i=1}^{T_k} \ee{i}_k\eqsp.
\]
We define $\tilde{Y}_k=\fraca{\sum_{i=1}^{T_k} \Y{i}_k}{\sigma_k\sqrt{T_k}}$, $\tilde{X}_k=\sqrt{T_k}X_k/\sigma_k$ and $\tilde{\eps}_k=\fraca{\sum_{i=1}^{T_k} \ee{i}_k}{\sigma_k \sqrt{T_k}}$ so that for all $k \in [K]$, $\tilde{Y}_k=\tilde{X}_k^T \bst + \tilde{\eps}_k$, where $\bE[\tilde{\eps}]=0$ and  $\Var[\tilde{\eps}_k]= 1$.
We denote by $\bX=(\tilde{X}_{1}^\top,\cdots,\tilde{X}_K^\top)\t\in \bR^{K\times d}$ the induced design matrix of the policy.
Under the assumption that  $\bX$ has full rank, the above Ordinary Least Squares (OLS) problem has an optimal unbiased estimator
$\bht=(\bX\t \bX)^{-1}\bX\t\tilde{Y}.$
The overarching objective is to upper-bound  $\bE \lVert\bht-\bst\rVert^2$, which can be easily rewritten as follows:
\[
\bE\left[\lVert\bht-\bst\rVert^2\right]=\Tr((\bX\t \bX)^{-1})=\Tr \left(\sum_{k=1}^K \tilde{X}_k \tilde{X}_k\t\right)\pinv=\dfrac{1}{T} \Tr\left(\sum_{k=1}^K p_k X_k X_k\t /\sig_k\right)\pinv \eqsp ,
\]
where  we have denoted for every $k \in [K]$, $p_k=T_k/T$ the proportion of times the covariate $X_k$ has been sampled. By definition,  $p=(p_1, \dots, p_K) \in \Dk$, the simplex of dimension $K-1$.
We emphasize here that minimizing $\bE \lVert\bht-\bst\rVert^2$ is equivalent to minimizing the trace of the inverse of the covariance matrix $\bX\t\bX$, which corresponds actually to A-optimal design~\citep{pukelsheim2006optimal}.
%We are now able to introduce the appropriate loss function of active linear regression.
Denote now by $\Omega(p)$ the following weighted covariance matrix
\[
\op=\sum_{k=1}^K \dfrac{p_k}{\sig_k} X_k X_k\t=\bX\t \bX \eqsp .
\]
The objective is to minimize over $p\in\Dk$ the loss function $L(p)=\Tr\left(\op\pinv\right)$ with $L(p) = +\infty $ if $(p \mapsto \Omega(p))$ is not invertible, such that
\[
\bE\left[\lVert\bht-\bst\rVert^2\right] = \dfrac{1}{T}\Tr\left(\op\pinv\right) = \dfrac{1}{T} L(p) \eqsp .
\]
For the problem  to be non-trivial, we require that the covariates span $\bR^d$. If it is not the case then there exists a vector along which one cannot get information about the parameter $\beta\st$. The best algorithm we can compare against can only estimate the projection of $\beta$ on the subspace spanned by the covariates, and we can work in this subspace.

The rest of this work is devoted to design an algorithm minimizing $\Tr\left(\op\pinv\right)$ with the difficulty that the variances $\sigma_k^2$ are unknown. In order to do that we will sequentially and adaptively choose which point to sample to minimize $\Tr\left(\op\pinv\right)$. This corresponds consequently to online A-optimal design.
As developed above, the norms of the covariates have a scaling role and those can be renormalized to lie on the sphere at no cost, which is thus an assumption from now on: $\forall k \in [K], \, \normd{X_k}=1$.
The following proposition shows that the problem we are considering is convex.

\begin{proposition}
\label{prop:cvx}
$L$ is strictly convex on $\Dd$ and continuous in its relative interior  $\mathring{\Dd}$.
\end{proposition}
The proof is deferred to Appendix~\ref{app:easy}.
Proposition \ref{prop:cvx} implies that
$L$ has a unique minimum $\pst$ in $\mathring{\Dd}$ and we note
\[
\displaystyle
\pst=\argmin_{p\in\Dd} L(p) \eqsp .
\]
Finally, we evaluate the performance of a sampling policy in term of ``regret'' \ie, the difference in loss between the optimal sampling policy and the policy in question.

\begin{definition}
Let $p_T$ denote the sampling proportions after $T$ samples of a policy. Its regret is then
\[
R(T)=\dfrac{1}{T}\left(L(p_T)-L(p\st)\right) \eqsp .
\]
\end{definition}
We will construct active sampling algorithms to minimize $R(T)$. A key step is the following computations of the gradient of $L$. Since
$\nabla_k \Omega(p)=X_k X_k^T/\sig_k$, it follows
\begin{align}
\partial_{p_k} L(p)&=-\dfrac{1}{\sig_k}\Tr \left( \Omega(p)^{-2} X_k X_k^T\right)=-\dfrac{1}{\sig_k}\norm{\Omega(p)^{-1}X_k}_2^2.
\end{align}
As in several works~\citep{hsu,allen2020} we will have to study different cases depending on the values of $K$ and $d$. The first one corresponds to the case $K \leq d$. As we explained it above, if $K<d$, the matrix $\op$ is not invertible and it is impossible to obtain a sublinear regret, which makes us work in the subspace spanned by the covariates $X_k$. This corresponds to $K=d$. We will treat this case in Sections~\ref{sec:algo_naive} and~\ref{sec:faster_algo}. The case $K>d$ is considered in Section~\ref{Se:KD}.
\section{A naive randomized algorithm}
\label{sec:algo_naive}

We begin by proposing an obvious baseline for the problem at hand.
One naive algorithm would be to estimate the variances of each of the covariates by sampling them a fixed amount of time. Sampling each arm $cT$ times (with $c<1/K$) would give an approximation $\hsig_k$ of $\sigma_k$ of order $1/\sqrt{T}$.
Then we can use these values to construct $\hop$ an approximation of $\op$ and then derive the optimal proportions $\hat{p}_k$ to minimize $\Tr(\hop\pinv)$.
Finally the algorithm would consist in using the remainder of the budget to sample the arms according to those proportions.
However, such a trivial algorithm would not provide good regret guarantees. Indeed the constant fraction $c$ of the samples used to estimate the variances has to be chosen carefully; it will lead to a $1/T$ regret if $c$ is too big (if $c > p_k\st$ for some $k$).
That is why we need to design an algorithm that will first roughly estimate the $p_k\st$. In order to improve the algorithm it will also be useful to refine at each iteration the estimates $\hat{p}_k$.
Following these ideas we propose Algorithm~\ref{algo:randomized} which uses a pre-sampling phase (see Lemma~\ref{lemma:presampling_active} for further details) and which constructs at each iteration lower confidence estimates of the variances, providing an optimistic estimate $\tilde{L}$ of the objective function $L$. Then the algorithm minimizes this estimate (with an offline A-optimal design algorithm, see \eg, \citep{optdesignoffline}). Finally the covariate $X_k$ is sampled with probability $\hp_{t,k}$. Then feedback is collected and estimates are updated.

\begin{algorithm}[h!]
		\caption{Naive randomized algorithm \label{algo:randomized}}
%\algsetup{indent=2em}
		\begin{algorithmic}[1]
	\REQUIRE $d$, $T$, $\delta$ confidence parameter
	\REQUIRE $N_1, \dots, N_d$ of sum $N$
	\STATE Sample $N_k$ times each covariate $X_k$
	\STATE $p_{N} \longleftarrow \left(N_1/N, \dots, N_d/N\right)$
	\STATE Compute empirical variances $\hsig^2_1, \dots, \hsig^2_d$
	\FOR{$N+1\leq t \leq T$}
	\STATE Compute $\hp_{t}\in \argmin~\tilde{L}$, where $\tilde{L}$ is the same function as $L$, but with variances replaced by lower confidence estimates of the variances (from Theorem~\ref{thm:var_conc}).
	\STATE Draw $\pi(t)$ randomly according to probabilities $\hp_{t}$ and sample covariate $X_{\pi(t)}$
	\STATE Update $p_{t+1}=p_t+\frac{1}{t+1}(e_{\pi(t+1)}-p_t)$ and $\hsig^2_{\pi(t)}$ where $(e_1, \dots, e_d)$ is the canonical basis of $\bR^d$.
	\ENDFOR
		\end{algorithmic}
		\end{algorithm}

\begin{proposition}\label{prop:randomized}
For $T \geq 1$ samples, running Algorithm~\ref{algo:randomized} with $N_i=p^o_i T/2$ (with $p^o$ defined by \eqref{eq:presampling_prop}) for all $i \in [K]$, gives final sampling proportions
$p_T$ such that
\[
 R(T) =\mathcal{O}_{\Gamma, \sigma_k}\pa{\dfrac{\sqrt{\log T}}{T^{3/2}}} \eqsp ,
 \]
 where $\Gamma$ is the Gram matrix of $X_1, \dots, X_K$.
\end{proposition}
The proof is postponed to Appendix~\ref{app:naive}.
Notice that we avoid the problem discussed by \citet{erraqabi2017trading} (that is due to infinite gradient on the simplex boundary)
thanks to presampling, allowing us to have positive empirical variance estimates with high probability.

\section{A faster first-order algorithm}
\label{sec:faster_algo}

We now improve the relatively ``slow'' dependency in $T$ in the rates of Algorithm~\ref{algo:randomized} -- due to its naive reduction to a MAB problem, and because it does not use any estimates of the gradient of $L$ -- with a different approach based on convex optimization techniques, that we can leverage to gain an order in the rates of convergence.

\subsection{Description of the algorithm}

The main algorithm is described in Algorithm~\ref{algo:fw} and is built following the work of~\citet{ucbfw}. The idea is to sample the arm sampled which minimizes the norm of a proxy of the gradient of $L$, corrected by a positive error term, as in the UCB algorithm~\citep{ucb}.
\begin{algorithm}[h!]%[thb]
	\caption{Bandit algorithm \label{algo:fw}}
%\algsetup{indent=2em}
	\begin{algorithmic}[1]
	\REQUIRE $d$, $T$%, $\delta$ confidence parameter
	\REQUIRE $N_1, \dots, N_d$ of sum $N$
	\STATE Sample $N_k$ times each covariate $X_k$
	\STATE $p_{N} \longleftarrow \left(N_1/N, \dots, N_d/N\right)$
	\STATE Compute empirical variances $\hsig^2_1, \dots, \hsig^2_d$
	\FOR{$N+1\leq t \leq T$}
	\STATE Compute $\nabla \hat{L}(p_t)$, where $\hat{L}$ is the same function as $L$, but with variances replaced by empirical variances.
	\FOR{$k \in [d]$}
	\STATE $\hat{g}_k \longleftarrow \nabla_k \hat{L}(p_t) -2\sqrt{\dfrac{3\log(t)}{T_k}}$
	\ENDFOR
	\STATE $\pi(t) \longleftarrow \argmin_{k \in [d]} \hat{g}_k$ and sample covariate $X_{\pi(t)}$
	\STATE Update $p_{t+1}=p_t+\frac{1}{t+1}(e_{\pi(t+1)}-p_t)$ and update $\hsig^2_{\pi(t)}$
	\ENDFOR
	\end{algorithmic}
\end{algorithm}
$N_1, \dots, N_d$ are the number of times each covariate is sampled at the beginning of the algorithm. This stage is needed to ensure that $L$ is smooth. More details about that will be given with Lemma~\ref{lemma:presampling_active}.

\subsection{Concentration of the gradient of the loss}

The cornerstone of the algorithm is to guarantee that the estimates of the gradients concentrate around their true value.
To simplify notations, we denote by $G_k=\partial_{p_k}L(p)$ the true k$^{\mbox{\scriptsize{th}}}$ derivative of $L$ and by $\hat{G}_k$ its estimate. More precisely, if we note $\hop=\sum_{k=1}^K (\fraca{p_k}{\hsig_k}) X_k X_k\t$, we have
\[
G_k=-\sigma_k^{-2}\|\op^{-1}X_k\|^2_2
\et
\hat{G}_k \doteq -\hsig_k^{-2}\|\hop^{-1}X_k\|^2_2 \eqsp .
\]
Since $\hat{G}_k$ depends on the $\hsig_k^2$, we need a concentration bound on the empirical variances $\hsig_k^2$. As traditional results on the concentration of the variances~\citep{empirical_bernstein, carpentier} are generally obtained in the bounded setting, we prove in Appendix~\ref{app:concentrate} the following bound in the case of subgaussian random variables.
\begin{theorem}
\label{thm:var_conc}
Let $X$ be a centered and $\kappa^2$-sub-gaussian random variable sampled $n \geq 2$ times. Let $\delta \in (0,1)$. Let $c=(e-1)(2e(2e-1))\pinv \approx 0.07$. With probability at least $1-\delta$, the following concentration bound on its empirical variance holds
\begin{align}
	\abs*{\hsig^2_n-\sigma^2} &\leq 3\kappa^2 \cdot \max\left(\dfrac{\log(4/\delta)}{cn},\sqrt{\dfrac{\log(4/\delta)}{cn}}\right) \eqsp .
\end{align}
\end{theorem}
Using Theorem~\ref{thm:var_conc} we claim the following concentration argument, which is the main ingredient of the analysis of Algorithm~\ref{algo:fw}.
\begin{proposition}
\label{prop:grad_conc} For every $k \in [K]$, after having gathered $T_k\leq T$ samples of covariates $X_k$, there exists a constant $\mathtt{C} >0$ (explicit and given in the proof) such that, with probability at least $1-\delta$
\begin{align}
\abs{G_k-\hat{G}_k} &\leq \mathtt{C} \left( \sigma_k\pinv \max_{i \in [K]} \dfrac{\sigma_i^2}{p_i} \right)^3 \cdot \max\left(\dfrac{\log(4TK/\delta)}{T_{k}},\sqrt{\dfrac{\log(4TK/\delta)}{T_{k}}}\right) \eqsp .
\end{align}
\end{proposition}
For clarity reasons we postpone the proof to Appendix~\ref{app:grad}.
Proving this proposition was one of the main technical challenges of our analysis. Now that we have it proven we can turn to the analysis of Algorithm~\ref{algo:fw}.

\subsection{Analysis of the convergence of the algorithm}

In convex optimization several classical assumptions can be leveraged to  derive fast convergence rates. Those assumptions are typically strong convexity, positive distance from the boundary of the constraint set, and  smoothness of the objective function, \ie, that it has Lipschitz gradient. We prove in the following that the loss $L$ satisfies them, up to the smoothness because its gradient explodes on the boundary of $\Dd$. However, $L$ is smooth on the relative interior of the simplex. Consequently we will circumvent this smoothness issue by using a technique from~\citep{regularized} consisting in pre-sampling every arm a linear number of times in order to force $p$ to be far from the boundaries of $\Dd$.

Using the following notations $\bX_0 \doteq (X_1\transpose, \cdots, X_d\transpose)\transpose$ and $\Gamma \doteq \bX_0\bX_0\t=\Gram(X_1,\dots,X_d)$ we prove the following lemma in Appendix~\ref{app:easyL}.
\begin{lemma}
\label{lemma:L}
The loss function $L$ verifies for all $p \in \Delta^d$,
\[
L(p)=\dfrac{1}{\det(\bX_0\t\bX_0)}\skd \dfrac{\sig_k}{p_k} \Cof(\bX_0\bX_0\t)_{kk} \eqsp .
\]
\end{lemma}
With this expression, the optimal proportion  $p\st$ can be easily computed using the KKT theorem, with the following closed form:
\begin{equation}
	\label{eq:pstar}
p\st_k=\fraca{\sigma_k \sqrt{\Cof(\Gamma)_{kk}}}{\sum_{i=1}^d \sigma_i \sqrt{\Cof(\Gamma)_{ii}}} \eqsp .
\end{equation}
This yields that $L$ is $\mu$-\textbf{strongly convex} on $\Dd$, with $\mu=2\det(\Gamma)\inv\min_i \Cof(\Gamma)_{ii}\sig_i$.
Moreover, this also implies that $p\st$ is \textbf{far away from the boundary} of $\Dd$.

\begin{lemma}
\label{lemma:eta}
Let $\eta \doteq \dist(p\st,\partial \Dd)$ be the distance from $p\st$ to the boundary of the simplex.
We have
\[
\eta=\sqrt{\dfrac{K}{K-1}}\dfrac{\min_i \sigma_i \sqrt{\Cof(\Gamma)_{ii}}}{\sum_{k=1}^d \sigma_k \sqrt{\Cof(\Gamma)_{kk}}} \eqsp .
\]
\end{lemma}
\begin{proof}
This is immediate with \eqref{eq:pstar} since $\eta = \sqrt{\dfrac{K}{K-1}}\min_i p_i\st$.
\end{proof}

It remains to recover the smoothness of $L$. This is done using a pre-sampling phase.
\begin{lemma}[see \citep{regularized}]
\label{lemma:presampling_active}
If there exists $\alpha \in (0,1/2)$ and $p^o \in \Dd$ such that $p\st \succcurlyeq \alpha p^o$ (component-wise) then
sampling arm $i$ at most $\alpha p_i^oT$ times (for all $i \in [d]$) at the beginning of the algorithm and running Algorithm~\ref{algo:fw}  is equivalent to running Algorithm~\ref{algo:fw} with budget $(1-\alpha)T$ on the smooth function $(p\mapsto L(\alpha p^o + (1-\alpha)p)$.
\end{lemma}
We have proved that $p\st_k$ is bounded away from 0 and thus a pre-sampling would be possible. However, this requires to have some estimate of each $\sig_k$. The upside is that those estimates must be accurate up to some multiplicative factor (and not additive factor) so that   a logarithmic number of samples of each arm is  enough to get valid  lower/upper bounds (see Corollary \ref{cor:var_conc}). Indeed, the estimate $\osig_k^2$ obtained satisfies,  for each $k \in [d]$, that $\sig_k \in [\osig_k^2/2, 3 \osig_k^2/2]$.  Consequently we know that
\[
\forall k \in [d], p_k\st \geq \dfrac{1}{\sqrt{3}}\dfrac{\osig_k \sqrt{\Cof(\Gamma)_{kk}}}{\sum_{i=1}^d \osig_i \sqrt{\Cof(\Gamma)_{ii}}} \geq \dfrac{1}{2} p^o, \quad \mbox{where } p^o=\dfrac{\osig_k \sqrt{\Cof(\Gamma)_{kk}}}{\sum_{i=1}^d \osig_i \sqrt{\Cof(\Gamma)_{ii}}}.\]
This will let us use Lemma~\ref{lemma:presampling_active} and with a presampling stage as prescribed,  $p$ is forced to remain  far away from the boundaries of the simplex in the sense that $ p_{t,i} \geq p_i^o/2$ at each stage $t $ subsequent to the pre-sampling, and for all $ i \in [d]$. Consequently, this logarithmic phase of estimation plus the linear phase of pre-sampling ensures that in the remaining of  the process, $L$ is actually smooth.
\begin{lemma}
\label{lemma:smooth_bandit}
With the pre-sampling of Lemma~\ref{lemma:presampling_active}, $L$ is smooth with constant $C_S$ where
\[
C_S \leq 432 \dfrac{  \sigma_{\max}^2 \left(\sum_{k=1}^d \sigma_k \sqrt{\Cof(\Gamma)_{kk}}\right)^3 }{ \det(\Gamma)\sigma_{\min}^3\sqrt{\min_k \Cof(\Gamma)_{kk}} } \eqsp .
\]
\end{lemma}
The proof is deferred to Appendix~\ref{app:proof_smooth}.
We can now state our main theorem that is proved in Appendix~\ref{app:proof_bandit}.
\begin{theorem}
\label{thm:rates}
Applying Algorithm~\ref{algo:fw} with $T \geq 1$ samples after having pre-sampled each arm $k \in [d]$ at most $p_k^oT/2$ times gives the following bound\footnote{The notation  $\bigo_{\Gamma, \sigma_k}$ means that there is a hidden constant depending on $\Gamma$ and on the $\sigma_k$. The explicit dependency on these parameters is given in the proof.}
\[
R(T)= \bigo_{\Gamma,\sigma_k}\pa{\dfrac{\log^2(T)}{T^{2}}} \eqsp .
\]
\end{theorem}
This theorem provides a fast convergence rate for the regret $R$ and emphasizes the importance of using the gradient information in Algorithm~\ref{algo:fw} compared to Algorithm~\ref{algo:randomized}.

\section{Discussion and generalization to $K>d$}\label{Se:KD}

We discuss in this section the case where the number $K$ of covariate vectors is greater than $d$.

\subsection{Discussion of the case $K>d$}
\label{ssec:ellipsoids}

In the case where $K>d$ it may be possible that the optimal $p\st$ lies on the boundary of the simplex $\Dk$, meaning that some arms should not be sampled. This happens for instance as soon as  there exist two covariate points that are exactly equal but with different variances. The point with the lowest variance should be sampled while the point with the highest one should not. All the difficulty of an algorithm for the case where $K>d$ is to be able to detect which covariate should be sampled and which one should not.
In order to adopt another point of view on this problem it might be interesting to go back to the field of optimal design of experiments.
Indeed by choosing $v_k=X_k/\sigma_k$, our problem consists exactly in the following constraint minimization problem given $v_1 \dots, v_K \in \bR^d$: 
\[
\tag{P}
\min\Tr\pa{\sum_{j=1}^K p_j v_j v_j\t}\inv\mbox{under contraints  }p \in \Delta^K \eqsp . \]
It is known (\citet{pukelsheim2006optimal}) that the dual problem of A-optimal design consists in finding the smallest ellipsoid, in some sense, containing all the points $v_j$:
\[
\tag{D}
\max\Tr(\sqrt{W})^2 \quad \mbox{under contraints }W \succ 0 \footnote{$W\succ0$ means here that $W$ is symmetric positive definite.} \mbox{ and }\ v_j\t Wv_j \leq 1 \mbox{ for all }1\leq j\leq K \eqsp .
\]
In our case the role of the ellipsoid can be easily seen with the KKT conditions. We obtain the following proposition, proved in Appendix~\ref{app:prop_ell}.
\begin{proposition}
\label{prop:ellipsoids}
The points $X_k/ \sigma_k$ lie within the ellipsoid defined by the matrix $\ops^{-2}$.
\end{proposition}
This geometric interpretation shows that a point $X_k$ with high variance is likely to be in the interior of the ellipsoid (because $X_k/\sigma_k$ is close to the origin), meaning that $\mu_k>0$ and therefore that $p_k\st=0$ \ie, that $X_k$ should not be sampled. Nevertheless since the variances are unknown, one is not easily able to find which point has to be sampled.
Figures illustrating the geometric interpretation can be found in Appendix~\ref{app:ellipsoids}.

\subsection{A theoretical upper-bound and a lower bound}

We derive now a bound for the convergence rate of Algorithm~\ref{algo:fw} in the case where $K>d$.
\begin{theorem}
\label{thm:Kd}
Applying Algorithm~\ref{algo:fw} with $K>d$ covariate points gives the following bound on the regret:
\[
R(T)=\bigo\pa{\log(T)T^{-5/4}} \eqsp .
\]
\end{theorem}
The proof is postponed to Appendix~\ref{app:Kgreaterd}.

One can ask whether this result is optimal, and if it is possible to reach the bound of Theorem~\ref{thm:rates}. The following theorem provides a lower bound showing that it is impossible in the case where there are $d$ covariates. However the upper and lower bounds of Theorems~\ref{thm:Kd} and~\ref{thm:lower} do not match. It is still an open question whether we can obtain better rates than $T^{-5/4}$.
\begin{theorem}
\label{thm:lower}
In the case where $K>d$, for any algorithm on our problem, there exists a set of parameters such that $R(T) \gtrsim T^{-3/2}$.
\end{theorem}
We prove Theorem~\ref{thm:lower} in Appendix~\ref{app:lower}.
\section{Numerical simulations}
We now present numerical experiments to validate our results and claims. We compare several algorithms for active matrix design: a very naive algorithm that samples equally each covariate, Algorithm~\ref{algo:randomized}, Algorithm~\ref{algo:fw} and a Thompson Sampling (TS) algorithm~\citep{thompson1933likelihood}. We run our experiments on synthetic data with horizon time $T$ between $10^4$ and $10^6$, averaging the results over $25$ rounds. We consider covariate vectors in $\bR^K$ of unit norm for values of $K$ ranging from $3$ to $100$. All the experiments ran in less than 15 minutes on a standard laptop.

Let us quickly describe the Thompson Sampling algorithm. We choose Normal Inverse Gamma distributions for priors for the mean and variance of each of the arms, as they are the conjugate priors for gaussian likelihood with unknown mean and variance.
At each time step $t$, for each arm $k \in [K]$, a value of $\hsig_k$ is sampled from the prior distribution. An approximate value of $\nabla_k{L}(p)$ is computed with the $\hsig_k$ values. The arm with the lowest gradient value is chosen and sampled. The value of this arm updates the hyperparameters of the prior distribution.

\begin{figure}[h!]
	\begin{center}
		\begin{minipage}[b]{0.48\linewidth}
			\scalebox{0.8}{
			\begin{tikzpicture}[domain=0:5]
					\begin{axis}
        [
        ,axis x line=bottom
  			,axis y line=center
				,ylabel near ticks,
				,xlabel near ticks,
        ,width=7.5cm
        ,xlabel=$\log(T)$
        ,ylabel=$\log(R(T))$
        %,xmin=0
        ,xmax=6.05
        %,ymin=0
        ,ymax=-2
        ,scaled ticks=false
        %,xtick={0,25000,50000,75000,100000}
        %,xticklabels={0,$\textrm{25,000}$,$\textrm{50,000}$,$\textrm{75,000}$,$\textrm{100,000}$}
        ,legend style={at={(0.5,0.7)},anchor=north west}
        ,cycle list name=exotic
        ]
        \addplot+[smooth,line width=1.2pt,each nth point={20}] table [x index={5}, y index={6}, col sep=comma] {file_base.csv};
        \addlegendentry{naive -- slope=$-1.0$};
        \addplot+[smooth,line width=1.2pt,each nth point={20}] table [x index={5}, y index={7}, col sep=comma] {file_base.csv};
        \addlegendentry{Alg.~\ref{algo:fw} -- slope=$-2.0$};
        \addplot+[smooth,line width=1.2pt,each nth point={20}] table [x index={5}, y index={8}, col sep=comma] {file_base.csv};
        \addlegendentry{TS -- slope=$-2.0$};
        \addplot+[smooth,line width=1.2pt,each nth point={20}] table [x index={5}, y index={10}, col sep=comma] {file_base.csv};
        \addlegendentry{Alg.~\ref{algo:randomized} -- slope=$-1.9$};
    		\end{axis}
			\end{tikzpicture}
			}
			\caption{Regret as a function of $T$ in log--log scale in the case of $K=3$ covariates in $\bR^3$.
			\label{fig:base_log}}
		\end{minipage}
	\hfill
	\begin{minipage}[b]{0.48\linewidth}
		\scalebox{0.8}{
		\begin{tikzpicture}[domain=0:5]
			\begin{axis}
        [
        ,axis x line=bottom
  			,axis y line=center
				,ylabel near ticks,
				,xlabel near ticks,
        ,width=7.5cm
        ,xlabel=$\log(T)$
        ,ylabel=$\log(R(T))$
        %,xmin=0
        %,xmax=800000
        %,ymin=0cs
        %,ymax=0.2e-7
        ,scaled ticks=false
        %,xtick={0,25000,50000,75000,100000}
        %,xticklabels={0,$\textrm{25,000}$,$\textrm{50,000}$,$\textrm{75,000}$,$\textrm{100,000}$}
        ,legend style={at={(0.5,0.7)},anchor=north west}
        ,cycle list name=exotic
        ]
        \addplot+[smooth,line width=1.2pt,each nth point={20}] table [x index={5}, y index={6}, col sep=comma] {file_K4d3.csv};
        \addlegendentry{naive -- slope=$-1.0$};
        \addplot+[smooth,line width=1.2pt,each nth point={20}] table [x index={5}, y index={7}, col sep=comma] {file_K4d3.csv};
        \addlegendentry{Alg.~\ref{algo:fw} -- slope=$-1.9$};
        \addplot+[smooth,line width=1.2pt,each nth point={20}] table [x index={5}, y index={8}, col sep=comma] {file_K4d3.csv};
        \addlegendentry{TS -- slope=$-1.9$};
    		\end{axis}
			\end{tikzpicture}
			}
\caption{Regret as a function of $T$ in log--log scale in the case of $K=4$ covariates in $\bR^3$. \label{fig:K4d3_log}}
		\end{minipage}
	\end{center}
\end{figure}
In our first experiment we consider only $3$ covariate vectors. We plot the results in log--log scale in order to see the convergence speed which is given by the slope of the plot. Results on Figure~\ref{fig:base_log} show that both Algorithms \ref{algo:randomized} and~\ref{algo:fw}, as well as Thompson sampling have regret $\bigo(1/T^2)$ as expected.

We see that Thompson Sampling performs well on low-dimensional data. However it is approximately $200$ times slower than Algorithm~\ref{algo:fw} -- due to the sampling of complex Normal Inverse Gamma distributions -- and therefore inefficient in practice.
On the contrary, Algorithm~\ref{algo:fw} is very practical. Indeed its computational complexity is  linear in time $T$ and its main computational cost is due to the computation of the gradient $\nabla \hat{L}$. This relies on  inverting   $\hat{\Omega} \in \bR^{d\times d}$, whose complexity is $\bigo(d^3)$  (or even $\bigo(d^{2.807})$ with Strassen algorithm). Thus the overall complexity of Algorithm~\ref{algo:fw} is $\bigo(T(d^{2.8}+K))$ hence polynomial. This computational complexity advocates  that Algorithm~\ref{algo:fw} is practical for moderate values of $d$, as in linear regression problems.

Figure~\ref{fig:base_log} shows that Algorithm~\ref{algo:randomized} performs nearly as well as Algorithm~\ref{algo:fw}. However, the minimization step of $\hat{L}$ is time-consuming when $K>d$, since there is no close form for $p\st$, which leads to approximate results. Therefore Algorithm~\ref{algo:randomized} is not adapted to  $K>d$.  We also have conducted similar experiments in this case, with $K=d+1$.    The offline solution of  the problem indicates that one covariate should not be sampled, \ie, $p\st \in \partial \Dk$. Results presented on Figure~\ref{fig:K4d3_log} prove the performances of Algorithm~\ref{algo:fw}.
\begin{figure}[h!]
	\begin{center}
		\begin{minipage}[b]{0.48\linewidth}
			\scalebox{0.8}{
			\begin{tikzpicture}[domain=0:5]
				\begin{axis}
        [
        ,axis x line=bottom
  			,axis y line=center
				,ylabel near ticks,
				,xlabel near ticks,
        ,width=7.5cm
        ,xlabel=$T$
        ,ylabel=$\log(R(T))$
        %,xmin=0
        %,xmax=800000
        %,ymin=0
        %,ymax=0.2e-7
        ,scaled ticks=false
        %,xtick={0,25000,50000,75000,100000}
        %,xticklabels={0,$\textrm{25,000}$,$\textrm{50,000}$,$\textrm{75,000}$,$\textrm{100,000}$}
        ,legend style={at={(0.5,0.6)},anchor=north west}
        ,cycle list name=exotic
        ]
        \addplot+[smooth,line width=1.2pt,each nth point={20}] table [x index={5}, y index={6}, col sep=comma] {file_challenge.csv};
        \addlegendentry{Alg.~\ref{algo:randomized} -- slope=$-1.0$};
        \addplot+[smooth,line width=1.2pt,each nth point={20}] table [x index={5}, y index={7}, col sep=comma] {file_challenge.csv};
        \addlegendentry{Alg.~\ref{algo:fw} -- slope=$-1.36$};
    		\end{axis}
			\end{tikzpicture}
			}
			\caption{Regret as a function of $T$ in log--log scale in the case of $K=4$ covariates in $\bR^3$ in a challenging setting.
			\label{fig:challenge_log}}
		\end{minipage}
		\hfill
		\begin{minipage}[b]{0.48\linewidth}
			\scalebox{0.8}{
			\begin{tikzpicture}[domain=0:5]
\begin{axis}
        [
        ,axis x line=bottom
  			,axis y line=center
				,ylabel near ticks,
				,xlabel near ticks,
        ,width=7.5cm
        ,xlabel=$\log(T)$
        ,ylabel=$\log(R(T))$
        ,xmin=3
        ,xmax=4.02
        %,ymin=0
        ,ymax=2
        %,scaled ticks=false
        %,xtick={0,25000,50000,75000,100000}
        %,xticklabels={0,$\textrm{25,000}$,$\textrm{50,000}$,$\textrm{75,000}$,$\textrm{100,000}$}
        ,legend style={at={(0.1,1.2)},anchor=north west}
        ,cycle list name=exotic
        ]
        \addplot+[smooth,line width=1.2pt,each nth point={15}] table [x index={1}, y index={2}, col sep=comma] {file_bandit_log_largeK.csv};
        \addlegendentry{$K=5\quad$ -- ~~ slope=$-1.98$};
        \addplot+[smooth,line width=1.2pt,each nth point={15}] table [x index={1}, y index={3}, col sep=comma] {file_bandit_log_largeK.csv};
        \addlegendentry{$K=10\quad$ -- ~~ slope=$-2.11$};
        \addplot+[smooth,line width=1.2pt,each nth point={15}] table [x index={1}, y index={4}, col sep=comma] {file_bandit_log_largeK.csv};
        \addlegendentry{$K=20\quad$ -- ~~ slope=$-2.23$};
        \addplot+[smooth,line width=1.2pt,each nth point={15}] table [x index={1}, y index={6}, col sep=comma] {file_bandit_log_largeK.csv};
        \addlegendentry{$K=50\quad$ -- ~~ slope=$-2.15$};
        \addplot+[smooth,line width=1.2pt,each nth point={15}] table [x index={1}, y index={7}, col sep=comma] {file_bandit_log_largeK.csv};
        \addlegendentry{$K=100\quad$ -- ~~ slope=$-2.06$};
    		\end{axis}
			\end{tikzpicture}
			}
			\caption{Regret as a function of $T$ for different values of $K$ in log--log scale.
			\label{fig:comparisonK}}
		\end{minipage}
	\end{center}
\end{figure}

One might argue that the positive results of Figure~\ref{fig:K4d3_log} are due to the fact that it is ``easy'' for the algorithm to detect that one covariate should not be sampled, in the sense that this covariate clearly lies in the interior of the ellipsoids mentioned in Section~\ref{ssec:ellipsoids}. In the very challenging case where two covariates  are equal but with variances separated by only $1/\sqrt{T}$, we obtain the results described on Figure~\ref{fig:challenge_log}. The observed experimental convergence rate is of the order of $T^{-1.36}$ which is much slower than the rates of Figure~\ref{fig:K4d3_log}, and between the rates proved in Theorems~\ref{thm:Kd} and Theorem~\ref{thm:lower}.

Finally we run a last experiment with larger values of $K=d$. We plot the convergence rate of Algorithm~\ref{algo:fw} for values of $K$ ranging from $5$ to $100$ in $\log-\log$ scale on Figure~\ref{fig:comparisonK}. The slope is again  approximately of $-2$, which is coherent with Theorem~\ref{thm:rates}. We note furthermore that  larger values of $d$ do not make  Algorithm~\ref{algo:fw} impracticable, as inferred by its cubic complexity.

% !TEX root = icml.tex

\section{Conclusion}

%Using bandit and convex optimization techniques we have proposed an algorithm solving the problem of linear regression with repeated queries. We obtain fast and optimal rates $\bigot(T^{-2})$ in the case where the covariates span a basis of $\bR^d$. Such fast rates cannot be obtained in the more general case of $K>d$ covariates, for which we derive weaker results. Extensive experiments illustrate our claims and the difficulty of this challenging setting.

We have proposed an algorithm mixing bandit and convex optimization techniques to solve the problem of online A-optimal design, which is related to active linear regression with repeated queries. This algorithm has proven fast and optimal rates $\bigot(T^{-2})$ in the case of $d$ covariates that can be sampled in $\bR^d$. One cannot obtain such fast rates in the more general case of $K>d$ covariates. We have therefore provided weaker results in this very challenging setting and conducted more experiments showing that the problem is indeed more difficult.

\bibliographystyle{apalike}
\bibliography{active_linear_regression}

\begin{thebibliography}{}

\bibitem[Agrawal and Devanur, 2014]{Agrawal_knapsacks}
Agrawal, S. and Devanur, N.~R. (2014).
\newblock Bandits with concave rewards and convex knapsacks.
\newblock In {\em Proceedings of the fifteenth ACM conference on Economics and
  computation}, pages 989--1006.

\bibitem[Allen-Zhu et~al., 2020]{allen2020}
Allen-Zhu, Z., Li, Y., Singh, A., and Wang, Y. (2020).
\newblock Near-optimal discrete optimization for experimental design: A regret
  minimization approach.
\newblock {\em Mathematical Programming}, pages 1--40.

\bibitem[Antos et~al., 2010]{antos2010active}
Antos, A., Grover, V., and Szepesv{\'a}ri, C. (2010).
\newblock Active learning in heteroscedastic noise.
\newblock {\em Theoretical Computer Science}, 411(29-30):2712--2728.

\bibitem[Atkeson and Alvarez, 2018]{PoolingBook}
Atkeson, L.~R. and Alvarez, R.~M. (2018).
\newblock {\em The Oxford handbook of polling and survey methods}.
\newblock Oxford University Press.

\bibitem[Auer et~al., 2002]{ucb}
Auer, P., Cesa-Bianchi, N., and Fischer, P. (2002).
\newblock Finite-time analysis of the multiarmed bandit problem.
\newblock {\em Mach. Learn.}, 47(2-3):235--256.

\bibitem[Berthet and Perchet, 2017]{ucbfw}
Berthet, Q. and Perchet, V. (2017).
\newblock Fast rates for bandit optimization with upper-confidence frank-wolfe.
\newblock In {\em Advances in Neural Information Processing Systems}, pages
  2225--2234.

\bibitem[Boyd and Vandenberghe, 2004]{boyd}
Boyd, S. and Vandenberghe, L. (2004).
\newblock {\em Convex optimization}.
\newblock Cambridge university press.

\bibitem[Carpentier et~al., 2011]{carpentier}
Carpentier, A., Lazaric, A., Ghavamzadeh, M., Munos, R., and Auer, P. (2011).
\newblock Upper-confidence-bound algorithms for active learning in multi-armed
  bandits.
\newblock In {\em International Conference on Algorithmic Learning Theory},
  pages 189--203. Springer.

\bibitem[Chafa{\"\i} et~al., 2012]{cmapx}
Chafa{\"\i}, D., Gu{\'e}don, O., Lecu{\'e}, G., and Pajor, A. (2012).
\newblock {\em Interactions between compressed sensing random matrices and high
  dimensional geometry}.
\newblock Citeseer.

\bibitem[Chen and Price, 2019]{pricechen}
Chen, X. and Price, E. (2019).
\newblock Active regression via linear-sample sparsification.
\newblock In Beygelzimer, A. and Hsu, D., editors, {\em Proceedings of the
  Thirty-Second Conference on Learning Theory}, volume~99 of {\em Proceedings
  of Machine Learning Research}, pages 663--695, Phoenix, USA. PMLR.

\bibitem[Cohn et~al., 1996]{jordan}
Cohn, D.~A., Ghahramani, Z., and Jordan, M.~I. (1996).
\newblock Active learning with statistical models.
\newblock {\em Journal of artificial intelligence research}, 4:129--145.

\bibitem[Cowan et~al., 2015]{cowan2015normal}
Cowan, W., Honda, J., and Katehakis, M.~N. (2015).
\newblock Normal bandits of unknown means and variances: Asymptotic optimality,
  finite horizon regret bounds, and a solution to an open problem.
\newblock {\em arXiv preprint arXiv:1504.05823}.

\bibitem[Derezi{\'n}ski et~al., 2019]{derezinski}
Derezi{\'n}ski, M., Warmuth, M.~K., and Hsu, D. (2019).
\newblock Unbiased estimators for random design regression.
\newblock {\em arXiv preprint arXiv:1907.03411}.

\bibitem[Erraqabi et~al., 2017]{erraqabi2017trading}
Erraqabi, A., Lazaric, A., Valko, M., Brunskill, E., and Liu, Y.-E. (2017).
\newblock {Trading off Rewards and Errors in Multi-Armed Bandits}.
\newblock In Singh, A. and Zhu, J., editors, {\em Proceedings of the 20th
  International Conference on Artificial Intelligence and Statistics},
  volume~54 of {\em Proceedings of Machine Learning Research}, pages 709--717,
  Fort Lauderdale, FL, USA. PMLR.

\bibitem[Fontaine et~al., 2019]{regularized}
Fontaine, X., Berthet, Q., and Perchet, V. (2019).
\newblock Regularized contextual bandits.
\newblock In Chaudhuri, K. and Sugiyama, M., editors, {\em Proceedings of
  Machine Learning Research}, volume~89 of {\em Proceedings of Machine Learning
  Research}, pages 2144--2153. PMLR.

\bibitem[Freund et~al., 1997]{selective_sampling}
Freund, Y., Seung, H.~S., Shamir, E., and Tishby, N. (1997).
\newblock Selective sampling using the query by committee algorithm.
\newblock {\em Machine learning}, 28(2-3):133--168.

\bibitem[Gao et~al., 2014]{optdesignoffline}
Gao, W., Chan, P.~S., Ng, H. K.~T., and Lu, X. (2014).
\newblock Efficient computational algorithm for optimal allocation in
  regression models.
\newblock {\em Journal of Computational and Applied Mathematics}, 261:118--126.

\bibitem[Goos and Jones, 2011]{DesignBook}
Goos, P. and Jones, B. (2011).
\newblock {\em Optimal design of experiments: a case study approach}.
\newblock John Wiley \& Sons.

\bibitem[Hazan and Karnin, 2014]{hardmargin}
Hazan, E. and Karnin, Z. (2014).
\newblock Hard-margin active linear regression.
\newblock In {\em International Conference on Machine Learning}, pages
  883--891.

\bibitem[Hsu et~al., 2011]{hsu}
Hsu, D., Kakade, S.~M., and Zhang, T. (2011).
\newblock An analysis of random design linear regression.
\newblock {\em arXiv preprint arXiv:1106.2363}.

\bibitem[Kirschner and Krause, 2018]{kirschner2018information}
Kirschner, J. and Krause, A. (2018).
\newblock Information directed sampling and bandits with heteroscedastic noise.
\newblock In Bubeck, S., Perchet, V., and Rigollet, P., editors, {\em
  Proceedings of the 31st Conference On Learning Theory}, volume~75 of {\em
  Proceedings of Machine Learning Research}, pages 358--384. PMLR.

\bibitem[Mannor et~al., 2014]{Mannor}
Mannor, S., Perchet, V., and Stoltz, G. (2014).
\newblock Approachability in unknown games: Online learning meets
  multi-objective optimization.
\newblock In {\em Conference on Learning Theory}, pages 339--355.

\bibitem[Maurer and Pontil, 2009]{empirical_bernstein}
Maurer, A. and Pontil, M. (2009).
\newblock Empirical bernstein bounds and sample variance penalization.
\newblock {\em arXiv preprint arXiv:0907.3740}.

\bibitem[Pukelsheim, 2006]{pukelsheim2006optimal}
Pukelsheim, F. (2006).
\newblock {\em Optimal design of experiments}.
\newblock SIAM.

\bibitem[Riquelme et~al., 2017a]{traceucb}
Riquelme, C., Ghavamzadeh, M., and Lazaric, A. (2017a).
\newblock Active learning for accurate estimation of linear models.
\newblock In Precup, D. and Teh, Y.~W., editors, {\em Proceedings of the 34th
  International Conference on Machine Learning}, volume~70 of {\em Proceedings
  of Machine Learning Research}, pages 2931--2939, International Convention
  Centre, Sydney, Australia. PMLR.

\bibitem[Riquelme et~al., 2017b]{threshold}
Riquelme, C., Johari, R., and Zhang, B. (2017b).
\newblock Online active linear regression via thresholding.
\newblock In {\em Thirty-First AAAI Conference on Artificial Intelligence}.

\bibitem[Sabato and Munos, 2014]{stratification}
Sabato, S. and Munos, R. (2014).
\newblock Active regression by stratification.
\newblock In Ghahramani, Z., Welling, M., Cortes, C., Lawrence, N.~D., and
  Weinberger, K.~Q., editors, {\em Advances in Neural Information Processing
  Systems 27}, pages 469--477. Curran Associates, Inc.

\bibitem[Sagnol, 2010]{sagnol2010optimal}
Sagnol, G. (2010).
\newblock {\em Optimal design of experiments with application to the inference
  of traffic matrices in large networks: second order cone programming and
  submodularity}.
\newblock PhD thesis, {\'E}cole Nationale Sup{\'e}rieure des Mines de Paris.

\bibitem[Soare, 2015]{soare}
Soare, M. (2015).
\newblock {\em {Sequential Resource Allocation in Linear Stochastic Bandits }}.
\newblock Theses, {Universit{\'e} Lille 1 - Sciences et Technologies}.

\bibitem[Sugiyama and Rubens, 2008]{modelselection}
Sugiyama, M. and Rubens, N. (2008).
\newblock Active learning with model selection in linear regression.
\newblock In {\em Proceedings of the 2008 SIAM International Conference on Data
  Mining}, pages 518--529. SIAM.

\bibitem[Thompson, 1933]{thompson1933likelihood}
Thompson, W.~R. (1933).
\newblock On the likelihood that one unknown probability exceeds another in
  view of the evidence of two samples.
\newblock {\em Biometrika}, 25(3/4):285--294.

\bibitem[Vershynin, 2018]{hdp}
Vershynin, R. (2018).
\newblock {\em High-dimensional probability: An introduction with applications
  in data science}, volume~47.
\newblock Cambridge University Press.

\bibitem[Wainwright, 2019]{wain}
Wainwright, M.~J. (2019).
\newblock {\em High-dimensional statistics: A non-asymptotic viewpoint},
  volume~48.
\newblock Cambridge University Press.

\bibitem[Wang and Chen, 2017]{wang2017improving}
Wang, Q. and Chen, W. (2017).
\newblock {Improving regret bounds for combinatorial semi-bandits with
  probabilistically triggered arms and its applications}.
\newblock In {\em Neural Information Processing Systems}.

\bibitem[Whittle, 1958]{whittle}
Whittle, P. (1958).
\newblock A multivariate generalization of tchebichev's inequality.
\newblock {\em The Quarterly Journal of Mathematics}, 9(1):232--240.

\bibitem[Yang et~al., 2013]{yang2013optimal}
Yang, M., Biedermann, S., and Tang, E. (2013).
\newblock On optimal designs for nonlinear models: a general and efficient
  algorithm.
\newblock {\em Journal of the American Statistical Association},
  108(504):1411--1420.

\end{thebibliography}

\appendix

\section{Concentration arguments}
\label{app:concentrate}

In this section we present results on the concentration of the variance for subgaussian random variables. Traditional results on the concentration of the variances~\citep{empirical_bernstein, carpentier} are obtained in the bounded setting. We propose results in a more general framework. Let us begin with some definitions.

\begin{definition}[Sub-gaussian random variable]
A random variable $X$ is said to be $\kappa^2$-sub-gaussian if
\[
\forall \lambda \geq 0, \ \exp(\lambda (X-\bE X)) \leq \exp(\lambda^2 \kappa^2 /2) \eqsp .
\]
And we define its $\psi_2$-norm as
\[
\normpsi{X}{2}=\inf\left\lbrace t>0 \, |\, \bE[\exp(X^2/t^2)]\leq 2\right\rbrace \eqsp .
\]
\end{definition}
We can bound the $\psi_2$-norm of a subgaussian random variable as stated in the following lemma.
\begin{lemma}[$\psi_2$-norm]
\label{lemma:psi2}
If $X$ is a centered $\kappa^2$-sub-gaussian random variable then
\[
\normpsi{X}{2}\leq \dfrac{2\sqrt{2}}{\sqrt{3}}\kappa \eqsp .
\]
\end{lemma}
\begin{proof}
A proposition from~\citep{wain} shows that for all $\lambda \in [0,1)$, a sub-gaussian variable $X$ verifies
\[
\bE \left(\dfrac{\lambda X^2}{2 \kappa^2} \right)\leq \dfrac{1}{\sqrt{1-\lambda}} \eqsp .
\]
Taking $\lambda=3/4$ and defining $u=\frac{2\sqrt{2}}{\sqrt{3}}\kappa$ gives
\[
\bE(X^2/u^2) \leq 2 \eqsp .
\]
Consequently $\normpsi{X}{2}\leq u$.
\end{proof}
A wider class of random variables is the class of sub-exponential random variables that are defined as follows.
\begin{definition}[Sub-exponential random variable]
A random variable $X$ is said to be sub-exponential if there exists $K>0$ such that
\[
\forall \ 0\leq\lambda\leq 1/K, \ \bE[\exp(\lambda \abs{X})]\leq \exp(K \lambda) \eqsp .
\]
And we define its $\psi_1$-norm as
\[
\normpsi{X}{1}=\inf\left\lbrace t>0 \, |\, \bE[\exp(\abs{X}/t)]\leq 2\right\rbrace \eqsp .
\]
\end{definition}
A result from~\citep{hdp} gives the following lemma, that makes a connection between subgaussian and subexponential random variables.
\begin{lemma}
\label{lemma:subeq}
A random variable $X$ is sub-gaussian if and only if $X^2$ is sub-exponential, and we have
\[
\normpsi{X^2}{1}=\normpsi{X}{2}^2 \eqsp .
\]
\end{lemma}
We now want to obtain a concentration inequality on the empirical variance of a sub-gaussian random variable. We give use the following notations to define the empirical variance.
\begin{definition}
We define the following quantities for $n$ \iid~repetitions of the random variable $X$.
\begin{align}
\mu &= \bE[X] \quad\et\quad \hmu_n = \dfrac{1}{n}\sum_{i=1}^n X_i \eqsp , \\
\mud &= \bE[X^2] \quad\et\quad \hmud_n = \dfrac{1}{n}\sum_{i=1}^n X_i^2 \eqsp .
\end{align}
The variance and empirical variance are defined as follows
\[
\sigma^2=\mud-\mu^2 \quad\et\quad \hsig^2_n=\hmud_n-\hmu^2_n \eqsp .
\]
\end{definition}
We are now able to prove \Cref{thm:var_conc} that we restate below for clarity.
\begin{theorem}
Let $X$ be a centered and $\kappa^2$-sub-gaussian random variable sampled $n \geq 2$ times. Let $\delta \in (0,1)$. Let $c=(e-1)(2e(2e-1))\pinv \approx 0.07$. With probability at least $1-\delta$, the following concentration bound on its empirical variance hold
\begin{align}
	\abs*{\hsig^2_n-\sigma^2} &\leq \dfrac{8}{3}\kappa^2 \cdot \max\left(\dfrac{\log(4/\delta)}{cn},\sqrt{\dfrac{\log(4/\delta)}{cn}}\right)+ 2\kappa^2 \dfrac{\log(4/\delta)}{n} \eqsp .
\end{align}
\end{theorem}

\begin{proof}
We have
\begin{align}
\abs*{\hsig^2_n-\sigma^2}&=\abs*{\hmud_n-\hmu^2_n-(\mud-\mu^2)} \\
&\leq \abs*{\hmud_n-\mud}+\abs*{\hmu^2_n-\mu^2} \\
&\leq \abs*{\hmud_n-\mud}+\abs*{\hmu_n-\mu}\abs*{\hmu_n+\mu}\\
&\leq \abs*{\hmud_n-\mud}+\abs*{\hmu_n}^2
\end{align}
since $\mu=0$.

We now apply Hoeffding's inequality to the $X_t$ variables that are $\kappa^2$-subgaussian, to get
\begin{align}
\bP \left( \dfrac{1}{n}\sinn X_i - \mu > t \right) &\leq \exp\left(-\dfrac{n^2 t^2}{2 n \kappa^2}\right)=\exp\left(-\dfrac{n t^2}{2 \kappa^2}\right).
\end{align}
And finally
\[
\bP \left(\abs*{\hmu_n-\mu} > \kappa \sqrt{ \dfrac{2 \log(2/\delta)}{n}} \right) \leq \delta.
\]
Consequently with probability at least $1-\delta$, $\abs*{\hmu_n}^2 \leq 2\kappa^2 \dfrac{\log(2/\delta)}{n}$.

The variables $X_t^2$ are sub-exponential random variables. We can apply Bernstein's inequality as stated in~\citep{cmapx} to get for all $t>0$:
\begin{align}
\bP\left(\abs*{\dfrac{1}{n}\sinn X_i^2 - \mud} > t\right) &\leq 2 \exp\left(-cn \min \left(\dfrac{t^2}{s^2},\dfrac{t}{m}\right)\right) \\
&\leq 2 \exp\left(-cn \min \left(\dfrac{t^2}{m^2},\dfrac{t}{m}\right)\right)\eqsp .
\end{align}
with $c=\frac{e-1}{2e(2e-1)}$, $s^2=\frac{1}{n}\sinn \norm{X_i^2}_{\psi_1}\leq m^2$ and $m=\max_{1\leq i\leq n} \norm{X_i^2}_{\psi_1}$.

Inverting the inequality we obtain
\[
\bP\left(\abs*{\hmud_n-\mud} > m \cdot \max\left(\dfrac{\log(2/\delta)}{cn},\sqrt{\dfrac{\log(2/\delta)}{cn}}\right)\right) \leq \delta \eqsp .
\]
And finally, with probability at least $1-\delta$,
\[
\abs*{\hsig^2_n-\sigma^2} \leq m \cdot \max\left(\dfrac{\log(4/\delta)}{cn},\sqrt{\dfrac{\log(4/\delta)}{cn}}\right) + 2\kappa^2 \dfrac{\log(4/\delta)}{n}\eqsp .
\]
Using Lemmas~\ref{lemma:subeq} and~\ref{lemma:psi2} we obtain that $m \leq 8 \kappa^2/3$.
Finally,
\begin{align}
\abs*{\hsig^2_n-\sigma^2} &\leq \dfrac{8}{3}\kappa^2 \cdot \max\left(\dfrac{\log(4/\delta)}{cn},\sqrt{\dfrac{\log(4/\delta)}{cn}}\right) + 2c\kappa^2 \dfrac{\log(4/\delta)}{cn} \\
&\leq 3\kappa^2 \cdot \max\left(\dfrac{\log(4/\delta)}{cn},\sqrt{\dfrac{\log(4/\delta)}{cn}}\right) \eqsp ,
\end{align}
since $2c \leq 1/3$.
This gives the expected result.
\end{proof}
We now state a corollary of this result.
\begin{corollary}
\label{cor:var_conc}
Let $T \geq 2$. Let $X$ be a centered and $\kappa^2$-sub-gaussian random variable. Let $c=(e-1)(2e(2e-1))\pinv \approx 0.07$. For $n=\ceil{\dfrac{72\kappa^4}{c\sigma^4}\log(2T)}$, we have with probability at least $1-1/T^2$,
\[
\abs*{\hsig^2_n-\sigma^2}\leq \dfrac{1}{2}\sigma^2.
\]
\end{corollary}

\begin{proof}
Let $\delta \in \ooint{0,1}$. 
Let $n=\ceil{\dfrac{\log(4/\delta)}{c}\parenthese{\dfrac{6\kappa^2}{\sigma^2}}^2}$.

Then $\dfrac{\log(4/\delta)}{cn}\leq\parenthese{\dfrac{\sigma^2}{6\kappa^2}}^2<1$, since $\sigma^2 \leq \kappa^2$, by property of subgaussian random variables.

With probability $1-\delta$, \Cref{thm:var_conc} gives
\begin{align}
\abs{\hsig^2_n-\sigma^2} \leq 3\kappa^2 \dfrac{\sigma^2}{6\kappa^2}\leq \dfrac{1}{2}\sigma ^2 \eqsp .
\end{align}
Now, suppose that $\delta=1/T^2$. Then, with probability $1-1/T^2$, for $n=\ceil{\dfrac{72\kappa^4}{c\sigma^4}\log(2T)}$ samples,
\begin{align}
\abs{\hsig^2_n-\sigma^2} \leq \dfrac{1}{2}\sigma ^2 \eqsp .
\end{align}
\end{proof}
\section{Proof of gradient concentration}
\label{app:grad}

In this section we prove Proposition~\ref{prop:grad_conc}.

\begin{proof}
Let $p \in \Dk$ and let $i \in [K]$. We compute
\begin{align}
G_i-\hat{G}_i &=\normd{\hop^{-1}\dfrac{X_i}{\hsig_i}}^2 - \normd{\op^{-1}\dfrac{X_i}{\sigma_i}}^2 \\
&\leq\normd{\hop^{-1}\dfrac{X_i}{\hsig_i}-\op^{-1}\dfrac{X_i}{\sigma_i}} \normd{\hop^{-1}\dfrac{X_i}{\hsig_i}+\op^{-1}\dfrac{X_i}{\sigma_i}}  \eqsp .
\end{align}
Let us now note $A \doteq \hop \hsig_i$ and $B \doteq \op \sigma_i$.
We have, using that $\normd{X_k}=1$,
\begin{align}
\normd{\hop^{-1}\dfrac{X_k}{\hsig_k}-\op^{-1}\dfrac{X_k}{\sigma_k}} &= \normd{(A\inv-B\inv)X_k} \\
&\leq \normd{A\inv-B\inv}\normd{X_k} \\
&\leq \normd{A\inv(B-A)B\inv} \\
&\leq \normd{A\inv} \normd{B\inv}\normd{B-A}.
\end{align}
One of the quantity to bound is $\normd{B\inv}$. We have
\[
\normd{B\inv}=\rho(B\inv)=\dfrac{1}{\min(\Sp(B))} \eqsp ,
\]
where $\Sp(B)$ is the spectrum (set of eigenvalues) of $B$.
We know that $\Sp(B)=\sigma_i \Sp(\op)$. Therefore we need to find the smallest eigenvalue $\lambda$ of $\op$. Since the matrix is invertible we know $\lambda>0$.

We will need the following lemma.
\begin{lemma}
\label{lemma:eig}
Let $\bX_0=\left(X_1\transpose, \cdots, X_k\transpose\right)\t$. We have
\[
\lm(\op) \geq \min_{k \in [K]} \dfrac{p_k}{\sigma_k^2} \lm(\bX_0\t\bX_0).
\]
\end{lemma}

\begin{proof}
We have for all $p\in \Dk$,
\[
\min_{i\in[K]} \dfrac{p_i}{\sigma_i^2}\sum_{k=1}^K X_k X_k\transpose \preccurlyeq \sum_{k=1}^K \dfrac{p_k}{\sigma_k^2} X_k X_k\transpose \eqsp .
\]
Therefore
\[
\min_{k\in[K]} \dfrac{p_k}{\sigma_k^2}\bX_0\transpose \bX_0 \preccurlyeq \op \eqsp .
\]
And finally \[\min_{k\in[K]} \dfrac{p_k}{\sigma_k^2} \lambda_{\min}(\bX_0\transpose \bX_0) \leq \lambda_{\min}(\op) \eqsp .
\]
\end{proof}
Note now that the smallest eigenvalue of $\bX_0\t\bX_0$ is actually the smallest non-zero eigenvalue of $\bX_0\bX_0\t$, which is the Gram matrix of $(X_1,\dots,X_d)$, that we note now $\Gamma$.
This directly gives the following
\begin{proposition}
\label{prop:normB}
\[\normd{B\inv} \leq \dfrac{1}{\sigma_i \lambda_{\min}(\Gamma)} \max_{k \in [K]} \dfrac{\sigma_k^2}{p_k}.\]
\end{proposition}
We jump now to the bound of $\normd{A\inv}$. We could obtain a similar bound to the one of $\normd{B\inv}$ but it would contain $\hsig_k$ values. Since we do not want a bound containing estimates of the variances, we prove the
\begin{proposition}
\label{prop:normA}
\[
\normd{A\inv}\leq 2 \normd{B\inv}.
\]
\end{proposition}
\begin{proof}
We have, if we note $H=A-B$,
\[
\normd{A\inv}=\normd{(B+A-B)\inv}\leq\normd{B\inv}\normd{(I_n+B\inv H)\inv}\leq 2 \normd{B\inv} \eqsp ,
\] from a certain rank.
\end{proof}
Let us now bound $\normd{B-A}$. We have
\begin{align}
\normd{B-A}&=\normd{\sigma_i \sum_{k=1}^K p_k\dfrac{\XX}{\sigma_k^2} -\hsig_i \skk p_k\dfrac{\XX}{\hsig_k^2}} \\
&=\normd{\skk p_k\XX \left(\dfrac{\sigma_i}{\sigma_k^2}-\dfrac{\hsig_i}{\hsig_k^2}\right)} \\
& \leq \skk p_k \abs*{\dfrac{\sigma_i}{\sigma_k^2}-\dfrac{\hsig_i}{\hsig_k^2}} \normd{X_k}^2 \\
& \leq \skk p_k \abs*{\dfrac{\sigma_i}{\sigma_k^2}-\dfrac{\hsig_i}{\hsig_k^2}}.
\end{align}
The next step is now to use Theorem~\ref{thm:var_conc} in order to bound the difference $ \abs*{\dfrac{\sigma_i}{\sigma_k^2}-\dfrac{\hsig_i}{\hsig_k^2}}$.
\begin{proposition}
\label{prop:ba}
With the notations introduced above, we have
\[
\normd{B-A} \leq \dfrac{113K\sigma_{\max}}{\sigma_{\min}^4} \kappa_{\max}^2 \cdot \max\left(\dfrac{\log(4TK/\delta)}{T_i},\sqrt{\dfrac{\log(4TK/\delta)}{T_i}}\right) \eqsp .
\]
\end{proposition}
\begin{proof}
Corollary~\ref{cor:var_conc} gives that for all $k \in [K]$, $\frac{1}{2}\sigma_k^2\leq\hsig_k^2\leq \frac{3}{2}\sigma_k^2$.

A consequence of Theorem~\ref{thm:var_conc} is that for all $k\in [K]$, if we note $T_k$ the (random) number of samples of covariate $k$, we have, with probability at least $1-\delta$,
\[
\forall k \in [K], \abs*{\sigma_k^2-\hsig_k^2} \leq \dfrac{8}{3}\kappa_k^2 \cdot \max\left(\dfrac{\log(4TK/\delta)}{cT_k},\sqrt{\dfrac{\log(4TK/\delta)}{cT_k}}\right) + 2\kappa_k^2 \dfrac{\log(4TK/\delta)}{T_k}.
\]
We note $\Delta_k$ the r.h.s of the last equation.
We begin by establishing a simple upper bound of $\Delta_k$.
Using the fact that $\sqrt{1/c} \leq 1/c$ and that $8/(3c) \leq 38$, we have
\begin{align}
\Delta_k &\leq \dfrac{8}{3c}\kappa_k^2 \cdot \max\left(\dfrac{\log(4TK/\delta)}{T_k},\sqrt{\dfrac{\log(4TK/\delta)}{T_k}}\right) + 2\kappa_k^2 \dfrac{\log(4TK/\delta)}{T_k} \\
&\leq 38\kappa_k^2 \cdot \max\left(\dfrac{\log(4TK/\delta)}{T_k},\sqrt{\dfrac{\log(4TK/\delta)}{T_k}}\right) + 2\kappa_k^2 \dfrac{\log(4TK/\delta)}{T_k} \\
&\leq 40\kappa_k^2 \cdot \max\left(\dfrac{\log(4TK/\delta)}{T_k},\sqrt{\dfrac{\log(4TK/\delta)}{T_k}}\right).
\end{align}
Let $k \in [K]$. We have
\begin{align}
\abs*{\dfrac{\sigma_i}{\sigma_k^2}-\dfrac{\hsig_i}{\hsig_k^2}} &= \abs*{\dfrac{\sigma_i\hsig_k^2-\hsig_i\sigma_k^2}{\sigma_k^2\hsig_k^2}}=\abs*{\dfrac{\sigma_i\hsig_k^2-\sigma_i\sigma_k^2+\sigma_i\sigma_k^2-\hsig_i\sigma_k^2}{\sigma_k^2\hsig_k^2}}\\
&\leq \abs*{\dfrac{\sigma_i(\hsig_k^2-\sigma_k^2)}{\sigma_k^2\hsig_k^2}} + \abs*{\dfrac{\sigma_i-\hsig_i}{\hsig_k^2}} \\
&\leq \abs*{\dfrac{\sigma_i(\hsig_k^2-\sigma_k^2)}{\sigma_k^2\hsig_k^2}} + \abs*{\dfrac{\sigma_i^2-\hsig_i^2}{\hsig_k^2(\sigma_i+\hsig_i)}} \\
&\leq \abs*{\dfrac{\sigma_i(\hsig_k^2-\sigma_k^2)}{\sigma_k^2\hsig_k^2}} + \abs*{\dfrac{\sigma_i^2-\hsig_i^2}{\hsig_k^2\sigma_i}} \\
&\leq \abs*{\hsig_k^2-\sigma_k^2}\abs*{\dfrac{\sigma_i}{\sigma_k^2\hsig_k^2}} + \abs*{\sigma_i^2-\hsig_i^2}\abs*{\dfrac{1}{\hsig_k^2\sigma_i}} \\
&\leq \Delta_k \dfrac{2\sigma_{\max}}{\sigma_{\min}^4}+\Delta_i \dfrac{2\sqrt{2}}{\sigma_{\min}^3}.
%&\leq 5\Delta_{\max}\dfrac{\sigma_{\max}}{\sigma_{\min}^4}.
\end{align}
Finally we have, using the fact that $T \geq T_k$ for all $k \in [K]$
\begin{align}
\normd{B-A}&\leq \skk p_k \abs*{\dfrac{\sigma_i}{\sigma_k^2}-\dfrac{\hsig_i}{\hsig_k^2}} \\
&\leq \dfrac{2\sigma_{\max}}{\sigma_{\min}^4}\left(\skk p_k\Delta_k+\sqrt{2}\skk p_k \Delta_i \right) \\
&\leq \dfrac{2\sigma_{\max}}{\sigma_{\min}^4}\left(\skk \dfrac{T_k}{T}40\kappa_k^2 \cdot \max\left(\dfrac{\log(4TK/\delta)}{T_k},\sqrt{\dfrac{\log(4TK/\delta)}{T_k}}\right)+\sqrt{2} \Delta_i \right) \\
%&\leq \dfrac{2\sigma_{\max}}{\sigma_{\min}^4}\left(\skk 40\kappa_k^2 \cdot \max\left(\dfrac{T_k}{T}\dfrac{\log(4TK/\delta)}{T_k},\dfrac{T_k}{T}\sqrt{\dfrac{\log(4TK/\delta)}{T_k}}\right)+\sqrt{2} \Delta_i \right) \\
&\leq \dfrac{2\sigma_{\max}}{\sigma_{\min}^4}\left(\skk 40\kappa_k^2 \cdot \max\left(\dfrac{\log(4TK/\delta)}{T},\sqrt{\dfrac{T_k}{T}}\sqrt{\dfrac{\log(4TK/\delta)}{T}}\right)+\sqrt{2} \Delta_i \right) \\
&\leq \dfrac{2\sigma_{\max}}{\sigma_{\min}^4}\left(\skk 40\kappa_k^2 \cdot \max\left(\dfrac{\log(4TK/\delta)}{T},\sqrt{\dfrac{\log(4TK/\delta)}{T}}\right)+\sqrt{2} \Delta_i \right) \\
%&\leq \dfrac{2\sigma_{\max}}{\sigma_{\min}^4}\left(K 40\kappa_{\max}^2 \cdot \max\left(\dfrac{\log(4TK/\delta)}{T},\sqrt{\dfrac{\log(4TK/\delta)}{T}}\right)+\sqrt{2} \Delta_i \right) \\
&\leq \dfrac{2\sigma_{\max}}{\sigma_{\min}^4}\left(K 40\kappa_{\max}^2 \cdot \max\left(\dfrac{\log(4TK/\delta)}{T_i},\sqrt{\dfrac{\log(4TK/\delta)}{T_i}}\right)+\sqrt{2} \Delta_i \right) \\
&\leq (K+\sqrt{2})\dfrac{80\sigma_{\max}}{\sigma_{\min}^4} \kappa_{\max}^2 \cdot \max\left(\dfrac{\log(4TK/\delta)}{T_i},\sqrt{\dfrac{\log(4TK/\delta)}{T_i}}\right).
\end{align}
\end{proof}
The last quantity to bound to end the proof is $\normd{\hop^{-1}\dfrac{X_k}{\hsig_k}+\op^{-1}\dfrac{X_k}{\sigma_k}}$.

\begin{proposition}
\label{prop:boundadd}
We have
\[
\normd{\hop^{-1}\dfrac{X_k}{\hsig_k}+\op^{-1}\dfrac{X_k}{\sigma_k}} \leq 3 \normd{B\inv}.
\]
\end{proposition}
\begin{proof}
We have
\begin{align}
\normd{\hop^{-1}\dfrac{X_k}{\hsig_k}+\op^{-1}\dfrac{X_k}{\sigma_k}} &= \normd{(A\inv+B\inv)X_k} \\
&\leq \normd{A\inv+B\inv}\normd{X_k} \\
&\leq \normd{(A\inv-B\inv) + 2 B\inv} \\
&\leq \normd{A\inv - B\inv} +2 \normd{B\inv}.
\end{align}
For $T$ sufficiently large we have $\normd{\hop^{-1}\dfrac{X_k}{\hsig_k}+\op^{-1}\dfrac{X_k}{\sigma_k}} \leq 3 \normd{B\inv}$.
\end{proof}
Combining Propositions~\ref{prop:normB},~\ref{prop:normA},~\ref{prop:ba} and~\ref{prop:boundadd} we obtain that $G_i-\hat{G}_i \leq 6 \normd{B\inv}^3\normd{B-A}$ and
\[
G_i-\hat{G}_i \leq 678 K \dfrac{\sigma_{\max}}{\sigma_{\min}^4} \left(\dfrac{1}{\sigma_i \lambda_{\min}(\Gamma)} \max_{k \in [K]} \dfrac{\sigma_k^2}{p_k}\right)^3 \cdot\kappa_{\max}^2 \cdot \max\left(\dfrac{\log(4TK/\delta)}{T_{i}},\sqrt{\dfrac{\log(4TK/\delta)}{T_{i}}}\right) \eqsp ,
\]
which proves Proposition~\ref{prop:grad_conc}.

\end{proof}
\section{Proofs of preliminary and easy results}
\label{app:easy}

In all the following we will denote by $\preccurlyeq$ the Loewner ordering: if $A$ and $B$ are two symmetric matrices, $A\preccurlyeq B$ iff $B-A$ is positive semi-definite.

\subsection{Proof of Proposition~\ref{prop:cvx}}

\begin{proof}
Let $p, q \in \mathring{\Dd}$, so that $\Omega(p)$ and $\Omega(q)$ are invertible,  and $\lambda \in [0,1]$. We have $L(p)=\Tr(\op\inv)$ and $L(\lambda p + (1-\lambda)q)=\Tr(\Omega(\lambda p + (1-\lambda)q)\inv)$, where
\begin{align}
\Omega(\lambda p + (1-\lambda q))&=\skd \dfrac{\lambda p_k + (1-\lambda) q_k}{\sigma_k^2}\XX \\
&=\lambda \Omega(p) + (1-\lambda) \Omega(q).
\end{align}
It is well-known \citep{whittle} that the inversion is strictly convex on the set of positive definite matrices.
Consequently,
\[
\Omega(\lambda p + (1-\lambda q)) \inv = \left(\lambda \Omega(p) + (1-\lambda) \Omega(q)\right)\inv \prec \lambda \op\inv + (1-\lambda)\Omega(q)\inv.
\]
Taking the trace this gives \[L(\lambda p + (1-\lambda)q) < \lambda L(p) + (1-\lambda)L(q).\]
Hence $L$ is convex.
\end{proof}

\subsection{Proof of Lemma~\ref{lemma:diag}}

\begin{lemma}
\label{lemma:diag}
Let $S$ be a symmetric positive definite matrix and $D$ a diagonal matrix with strictly positive entries $d_1, \dots, d_n$. Then \[\lambda_{\min}(DSD) \geq \min_i(d_i)^2 \lambda_{\min}(S).\]
\end{lemma}

\begin{proof}
We have $\lambda_{\min}(S)Id \preccurlyeq S$ and consequently, multiplying by $D$ (positive definite) to the right and left we obtain $\lambda_{\min}(S)D^2\preccurlyeq DSD$, hence
\[
\min_i(d_i)^2 \lambda_{\min}(S)\leq \lambda_{\min}(DSD) .
\]
\end{proof}

\section{Proofs of the slow rates}
\label{app:naive}

\subsection{Proof of Proposition~\ref{prop:randomized}}

\begin{proof}
We now conduct the analysis of Algorithm~\ref{algo:randomized}. Our strategy will be to convert the error $L(p_T)-L(p\st)$ into a sum over $t\in [T]$ of small errors.
Notice first that the quantity \[{\norm{{\Omega}(p)^{-1}{X_k}{}}_2^2}\]
can be upper bounded by $\dfrac{1}{\sigma_i \lambda_{\min}(G)} \max_{k \in [K]} \dfrac{\sigma_k^2}{0.5 p^o}$, for  $p=p_T$. For $p=\hat p_t$, we can also bound this quantity by $\dfrac{4}{\sigma_i \lambda_{\min}(G)} \max_{k \in [K]} \dfrac{\sigma_k^2}{0.5 p^o}$, using Lemma~\ref{lemma:presampling_active} to express $\hat p_t$ with respect to lower estimates of the variances --- and thus with respect to real variance thanks to Corollary~\ref{cor:var_conc}. Then, from the convexity of $L$, we have

\begin{align}
 L(p_T)-L(p\st)&=L(p_T)-L\pa{1/T\sum_{t=1}^T \hat p_t}+ L\pa{\dfrac{1}{T}\sum_{t=1}^T \hat p_t}-L(p\st)\\
 &\leq \sum_k-\norm{{\Omega}(p_T)^{-1}\dfrac{X_k}{{\sigma}_k}}_2^2\pa{ p_{k,T}-\dfrac{1}{T}\sum_{t=1}^T \hat p_{k,t}} + \dfrac{1}{T}\sum_{t=1}^T \pa{L(\hat p_t)-L(p\st)} \\
\end{align}

Using Hoeffding inequality, $\pa{ p_{k,T}-\frac{1}{T}\sum_{t=1}^T \hat p_{k,t}}=\frac{1}{T}\sum_{t=1}^T\pa{ \II{k\text{ is sampled at }t}-\hat p_{k,t}}$ is bounded by $\sqrt{\frac{\log(2/\delta)}{T}}$ with probability $1-\delta$. It thus remains to bound the second term $\frac{1}{T}\sum_{t=1}^T \pa{L(\hat p_t)-L(p\st)}$. First, notice that $L(p)$ is an increasing function of $\sigma_i$ for any $i$. If we define $\hat L$ be replacing each $\sigma_i^2$ by lower confidence estimates of the variances $\tilde\sigma_i^{2}$ (see Theorem~\ref{thm:var_conc}), then
\begin{align}
 L(\hat p_t)-L(p\st)\leq L(\hat p_t)-\hat L(p\st) =  L(\hat p_t) - \hat L(\hat p_t) + \hat L(\hat p_t)-\hat L(p^*)\leq  L(\hat p_t) - \hat L(\hat p_t).
\end{align}

Since the gradient of $L$ with respect to $\sigma^2$ is $\pa{\frac{2p_i}{\sigma_i^3}\norm{{\Omega}(p)^{-1}{X_i}{}}_2^2}_i$, we can bound $L(\hat p_t) - \hat L(\hat p_t)$ by
\[ 1/\sigma_{\min}^3 \sup_{k}{\norm{{\Omega}(\hat p_t)^{-1}{X_k}{}}_2^2} \sum_i 2\hat p_{i,t} \abs{\sigma^2_i - \tilde\sigma_i^{2}} .\]

Since $\hat p_{i,t}$ is the probability of having a feedback from covariate $i$, we can use the probabilistically triggered arm setting of \citet{wang2017improving} to prove that $\dfrac{1}{T}\sum_{t=1}^T\sum_i 2\hat p_i \abs{\sigma^2_i - \tilde\sigma_i^{2}}=
\mathcal{O}\pa{\sqrt{\frac{\log(T)}{T}}}$. Taking $\delta$ of order $T^{-1 }$ gives the desired result.
\end{proof}

\section{Analysis of the bandit algorithm}
\label{app:bandit}

\subsection{Proof of Lemma~\ref{lemma:L}}
\label{app:easyL}

We begin by a lemma giving the coefficients of $\op\inv$.
\begin{lemma}
\label{lemma:coefs}
The diagonal coefficients of $\op\inv$ can be computed as follows:
\[\forall i \in [d], \
\op\inv_{ii}=\sum_{j=1}^d \dfrac{\sig_j \Cof(\bX_0\t)_{ij}^2}{\det(\bX_0^T\bX_0)}\dfrac{1}{p_j} \eqsp .\]
\end{lemma}
\begin{proof}
We suppose that $\forall i \in [d], \ p_i \neq 0$ so that $\op$ is invertible.

We know that $\Omega(p)^{-1}=\dfrac{\mbox{Com}(\Omega(p))\t}{\det(\Omega(p))}$. We compute now $\det(\op)$.
\begin{align}
\det(\op)&=\det\left(\sum_{k=1}^d \dfrac{p_k X_k X_k\t}{\sig_k}\right)
=\det((\sqrt{T\inv}\bX)\t \sqrt{T\inv}\bX)
=T^{-d}\det(\bX\t)^2\\
&=T^{-d}\begin{vmatrix}
 & \vdots &  \\
\tilde{X}_1 & \vdots & \tilde{X}_d\\
 & \vdots & \
\end{vmatrix}^2
=\begin{vmatrix}
 & \vdots &  \\
\dfrac{\sqrt{p_1}}{\sigma_1} X_1 & \vdots & \dfrac{\sqrt{p_d}}{\sigma_d} X_d\\
 & \vdots & \
\end{vmatrix}^2\\
&=\det(\bX_0)^2 \dfrac{p_1}{\sig_1} \cdots \dfrac{p_d}{\sig_d} \eqsp .
\end{align}
We now compute $\Com(\op)_{ii}$.
\[
\Com(\op)=\Com(T^{-1/2}\bX\t T^{-1/2}\bX)=\Com(T^{-1/2}\bX\t)\Com(T^{-1/2}\bX\t)\t \eqsp .
\]
Let us note $M \doteq T^{-1/2}\bX =\begin{pmatrix}
\cdots & \dfrac{\sqrt{p_1}}{\sigma_1}X_1\t & \cdots \\
& \vdots & \\
\cdots & \dfrac{\sqrt{p_K}}{\sigma_K}X_K\t & \cdots
\end{pmatrix}$.
Therefore
\begin{align}
\Com(\op)_{ii}&=\sum_{j=1}^d \Com(M\t)_{ij}^2=\sum_{j=1}^d \prod_{k \neq j}\dfrac{p_k}{\sig_k} \Cof(\bX_0\t)_{ij}^2 \eqsp .
\end{align}
Finally,
\[
\op\inv_{ii}=\sum_{j=1}^d \dfrac{\sig_j \Cof(\bX_0\t)_{ij}^2}{\det(\bX_0\t\bX_0)}\dfrac{1}{p_j} \eqsp .
\]
\end{proof}
This allows us to derive the exact expression of the loss function $L$ and we restate \Cref{lemma:L}.
\begin{lemma}
We have, for all $p \in \Delta^d$,
\[
L(p)=\dfrac{1}{\det(\bX_0\t\bX_0)}\skd \dfrac{\sig_k}{p_k} \Cof(\bX_0\bX_0\t)_{kk} \eqsp .
\]
\end{lemma}

\begin{proof}
Using Lemma~\ref{lemma:coefs} we obtain
\begin{align}
L(p)&=\Tr(\op\inv)=\sum_{k=1}^d \op\inv_{kk}\\
 &=\dfrac{1}{\det(\bX\t\bX)}\skd \dfrac{\sig_k}{p_k} \sum_{i=1}^d \Cof(\bX_0\t)_{ik}^2 =\dfrac{1}{\det(\bX_0\t\bX_0)}\skd \dfrac{\sig_k}{p_k} \Com(\bX_0\bX_0\t)_{kk} \eqsp .
\end{align}
\end{proof}

\subsection{Proof of \Cref{lemma:smooth_bandit}}
\label{app:proof_smooth}

\begin{proof}
We use the fact that for all $i \in [d]$, $p_i \geq p^o_i/2$.
We have that for all $i\in[d]$, \[\nabla^2_{ii} L(p)=\dfrac{\Cof(\Gamma)_{ii}\sig_i}{\det(\Gamma)}\dfrac{2}{p_i^3}\leq \dfrac{2\ \Cof(\Gamma)_{ii}\sig_i}{\det(\Gamma) (p^o_i/2)^3}.\]
We have $p^o_k=\dfrac{\osig_k \sqrt{\Cof(\Gamma)_{kk}}}{\sum_{i=1}^d \osig_i \sqrt{\Cof(\Gamma)_{ii}}}$ which gives
\[
\nabla^2_{ii} L(p) \leq 16 \dfrac{\sigma_{\max}^2 \left(\sum_{k=1}^d \osig_k \sqrt{\Cof(\Gamma)_{kk}}\right)^3}{\det(\Gamma)\osig_{\min}^3\sqrt{\min_k \Cof(\Gamma)_{kk}}}\doteq C_S.
\]
And consequently $L$ is $C_S$-Lipschitz smooth.

We can obtain an upper bound on $C_S$ using Corollary~\ref{cor:var_conc}, which tells that $\sigma_k/2\leq \osig_k \leq 3\sigma_k/2$:
\[
C_S \leq 432 \dfrac{\sigma_{\max}^2 \left(\sum_{k=1}^d \sigma_k \sqrt{\Cof(\Gamma)_{kk}}\right)^3}{\det(\Gamma)\sigma_{\min}^3\sqrt{\min_k \Cof(\Gamma)_{kk}}}.
\]
\end{proof}

\subsection{Proof of \Cref{thm:rates}}
\label{app:proof_bandit}

\begin{proof}
Proposition~\ref{prop:grad_conc} gives that
\[
\abs{G_i-\hat{G}_i} \leq 678 K \dfrac{\sigma_{\max}}{\sigma_{\min}^4} \left(\dfrac{1}{\sigma_i \lambda_{\min}(\Gram)} \max_{k \in [K]} \dfrac{\sigma_k^2}{p_k}\right)^3 \cdot\kappa_{\max}^2 \cdot \max\left(\dfrac{\log(4TK/\delta)}{T_{i}},\sqrt{\dfrac{\log(4TK/\delta)}{T_{i}}}\right).
\]
Since each arm has been sampled at least a linear number of times we guarantee that $\log(4TK/\delta)/T_i \leq 1$ such that
\[
\abs{G_i-\hat{G}_i} \leq 678 K \pa{\dfrac{\sigma_{\max}}{\sigma_{\min}}}^7\dfrac{1}{\lambda_{\min}(\Gamma)^3} \dfrac{\kappa_{\max}^2}{p_{\min}^3} \sqrt{\dfrac{\log(4TK/\delta)}{T_{i}}}.
\]
Thanks to the presampling phase of Lemma~\ref{lemma:presampling_active}, we know that $p_{\min} \geq p^o/2$.
For the sake of clarity we note $C \doteq 678 K \pa{\dfrac{\sigma_{\max}}{\sigma_{\min}}}^7\dfrac{8}{{p^o}^3\lambda_{\min}(\Gamma)^3} \kappa_{\max}^2$ such that $\abs{G_i-\hat{G}_i} \leq C \sqrt{\dfrac{\log(4TK/\delta)}{T_{i}}}$.

We have seen that $L$ is $\mu$-strongly convex, $C_L$-smooth and that $\dist(p\st,\partial\Dd) \geq \eta$. Consequently, since Lemma~\ref{lemma:presampling_active} shows that the pre-sampling stage does not affect the convergence result, we can apply~\citep[Theorem 7]{ucbfw} (with the choice $\delta_T=1/T^2$, which gives that
\[
\bE[L(p_T)] - L(p\st) \leq c_1 \dfrac{\log^2(T)}{T}+c_2 \dfrac{\log(T)}{T}+c_3\dfrac{1}{T} \eqsp ,
\]
with $c_1=\dfrac{96C^2K}{\mu \eta^2}$, $c_2=\dfrac{24C^2}{\mu\eta^3}+S$ and $c_3=\dfrac{3072^2K}{\mu^2\eta^4}\normi{L}+\dfrac{\mu\eta^2}{2}+C_S$.
With the presampling stage and Lemma~\ref{lemma:L}, we can bound $\normi{L}$ by
\[
\normi{L} \leq \dfrac{\sum_j \sig_j \Cof(\Gamma)_{jj}}{\sigma_{\min}\sqrt{\Cof(\Gamma)}_{\min}}\pa{\sum_j \sigma_j \sqrt{\Cof(\Gamma)_{jj}}} \eqsp .
\]
We conclude the proof using the fact that $R(T)=\dfrac{1}{T}\left(L(p_T)-L(p\st)\right)$.
\end{proof}

\section{Analysis of the case $K>d$}

\subsection{Proof of \Cref{thm:Kd}}
\label{app:Kgreaterd}

\begin{proof}
In order to ensure that $L$ is smooth we pre-sample each covariate $n$ times. We note $\alpha=n/T \in (0,1)$. This forces $p_i$ to be greater than $\alpha$ for all $i$. Therefore $L$ is $C_S$-smooth with $C_S \leq \dfrac{2\max_k\Cof(\Gamma)_{kk}\sigma_{\max}^2}{\alpha^3 \det(\Gamma)}\doteq \dfrac{C}{\alpha^3}$.

We use a similar analysis to the one of~\citep{ucbfw}. Let us note $\rho_t \doteq L(p_t)-L(p\st)$ and $\varepsilon_{t+1}\doteq (e_{\pi(t+1)}-e_{\star_{t+1}})\t\nabla L(p_t)$ with $e_{\star_{t+1}}=\argmax_{p \in \Delta^K} p\t \nabla L(p_t)$.
\citep[Lemma 12]{ucbfw} gives for $t \geq nK$,
\[
(t+1)\rho_{t+1} \leq t\rho_t+\varepsilon_{t+1}+\dfrac{C_S}{t+1}.
\]
Summing for $t\geq nK$ gives
\begin{align}
T\rho_T &\leq nK\rho_{nK} + C_S \log(eT)+\sum_{t=nK}^T \varepsilon_t \\
L(p_T)-L(p\st) &\leq K \alpha (L(p_{nK})-L(p\st)) + \dfrac{C}{\alpha^3}\dfrac{\log(eT)}{T} + \dfrac{1}{T}\sum_{t=nK}^T \varepsilon_t \eqsp .
\end{align}

We bound $\sum_{t=nK}^T \varepsilon_t/T$ as in Theorem 3 of~\cite{ucbfw} by $4\sqrt{\dfrac{3K\log(T)}{T}}+\pa{\dfrac{\pi^2}{6}+K}\dfrac{2\normi{\nabla L}+\normi{L}}{T}=\bigo\pa{\sqrt{\dfrac{\log(T)}{T}}}$.

We are now interested in bounding $\alpha (L(p_{nK})-L(p\st))$.

By convexity of $L$ we have
\[
L(p_{nK})-L(p\st) \leq \la \nabla L(p_{nK}), p_{nK}-p\st \ra \leq \normd{\nabla L(p_{nK})} \normd{p_{nK}-p\st} \leq 2 \normd{\nabla L(p_{nK})}.
\]
We have also
\[
\dpart{L}{p_k}(p_{nK})=-\normd{\Omega(p_{nK})\inv\dfrac{X_k}{\sigma_k}}^2 \eqsp .
\]
Proposition~\ref{prop:normB} shows that
\[
\normd{\op\inv}\leq \dfrac{1}{\lambda_{\min}(\Gamma)}\dfrac{\sigma_{\max}^2}{\min_k p_k} \eqsp .
\]
In our case, $\min_k p_{nK}=1/K$. Therefore
\[
\normd{\Omega(p_{nK})\inv}\leq \dfrac{K\sigma_{\max}^2}{\lambda_{\min}(\Gamma)} \eqsp .
\]
And finally we have
\[
\normd{\nabla L(p_{nK})}\leq \dfrac{K}{\sqrt{\lambda_{\min}(\Gamma)}}\dfrac{\sigma_{\max}}{\sigma_{\min}} \eqsp .
\]
We note $C_1\doteq \dfrac{2K^2}{\sqrt{\lambda_{\min}(\Gamma)}}\dfrac{\sigma_{\max}}{\sigma_{\min}}$.
This gives
\[
L(p_T)-L(p\st) \leq \alpha C_1 + \dfrac{C}{\alpha^3}\dfrac{\log(T)}{T}+\bigo\pa{\sqrt{\dfrac{\log(T)}{T}}}\eqsp .
\]
The choice of $\alpha = T^{-1/4}$ finally gives
\[
L(p_T)-L(p\st)=\bigo\pa{\dfrac{\log(T)}{T^{1/4}}} \eqsp .
\]
\end{proof}

\subsection{Proof of \Cref{thm:lower}}
\label{app:lower}

\begin{proof}
For simplicity we consider the case where $d=1$ and $K=2$. Let us suppose that there are two points $X_1$ and $X_2$ that can be sampled, with variances $\sigma_1^2=1$ and $\sigma_2^2=1+\Delta>1$, where $\Delta \leq 1$. We suppose also that $X_1=X_2=1$ such that both points are identical.

The loss function associated to this setting is
\[
L(p)=\pa{\dfrac{p_1}{\sigma_1^2}+\dfrac{p_2}{\sigma_2^2}}\inv=\dfrac{1+\Delta}{p_2+p_1(1+\Delta)}=\dfrac{1+\Delta}{1+\Delta p_1}.
\]
The optimal $p$ has all the weight on the first covariate (of lower variance): $p\st=(1,0)$ and $L(p\st)=1$.

Therefore
\[L(p)-L(p\st)=\dfrac{1+\Delta}{1+\Delta p_1}-1=\dfrac{p_2\Delta}{1+\Delta p_1} \geq \dfrac{\Delta}{2}p_2 \eqsp .
\]

We see that we are now facing a classical 2-arm bandit problem: we have to choose between arm $1$ giving expected reward $0$ and arm $2$ giving expected reward $\Delta/2$. Lower bounds on multi-armed bandits problems show that
\[
\bE L(p_T)-L(p\st) \gtrsim \dfrac{1}{\sqrt{T}} \eqsp .
\]
Thus we obtain \[R(T) \gtrsim \dfrac{1}{T^{3/2}} \eqsp .\]
\end{proof}
% !TEX root = supplementary.tex

%\newpage

\section{Geometric Interpretation}

\subsection{Proof of Proposition~\ref{prop:ellipsoids}}
\label{app:prop_ell}

\begin{proof}
We want to minimize $L$ on the simplex $\Dk$. Let us introduce the Lagrangian function
\[
\sL:(p_1,\dots,p_K,\lambda,\mu_1,\dots,\mu_K) \in \bR^K\times\bR\times\bR_+^K \mapsto L(p)+\lambda \parenthese{\skk p_k-1} - \la \mu, p \ra
\]
Applying Karush-Kuhn-Tucker theorem gives that $p\st$ verifies
\[
\forall k \in [d], \ \dpart{\sL}{p_k}(p\st)=0.
\]
Consequently
\[
\forall k \in [d], \ \normd{\ops\inv\dfrac{X_k}{\sigma_k}}^2=\lambda-\mu_k \leq \lambda.
\]
This shows that the points $X_k/ \sigma_k$ lie within the ellipsoid defined by the equation $x\t \ops^{-2}x \leq \lambda$.

\end{proof}

\subsection{Geometric illustrations}
\label{app:ellipsoids}

In this section we present figures detailing the geometric interpretation discussed in Section~\ref{Se:KD}.

\begin{figure}[h!]
 \begin{minipage}[b]{.43\linewidth}
  \centering
  %\subfigure[$p_1=0.21 \quad p_2=0.37 \quad p_3=0.42$ \label{fig:circle_all}]{\includegraphics[width=0.8\textwidth]{circle_all.jpg}}
  \includegraphics[width=0.8\textwidth]{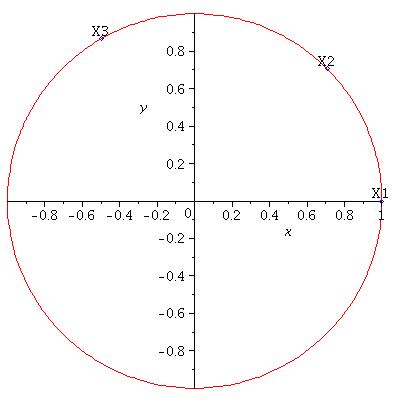}
  \subcaption{$p_1=0.21 \quad p_2=0.37 \quad p_3=0.42$ \label{fig:circle_all}}
    \vspace*{1em}
 \end{minipage} \hfill
 \begin{minipage}[b]{.43\linewidth}
  \centering
  %\subfigure[$p_1=0 \quad p_2=0.5 \quad p_3=0.5$ \label{fig:circle0}]{\includegraphics[width=0.8\textwidth]{circle_zero.jpg}}
  \includegraphics[width=0.8\textwidth]{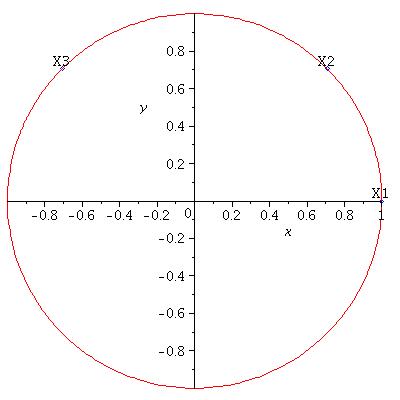}
  \subcaption{$p_1=0 \quad p_2=0.5 \quad p_3=0.5$ \label{fig:circle0}}
  \vspace*{1em}
 \end{minipage}
 \begin{minipage}[b]{.43\linewidth}
  \centering
  %\subfigure[$p_1=0.5 \quad p_2=0 \quad p_3=0.5$ \label{fig:ellipse0}]{\includegraphics[width=0.8\textwidth]{ellipse.jpg}}
  \includegraphics[width=0.8\textwidth]{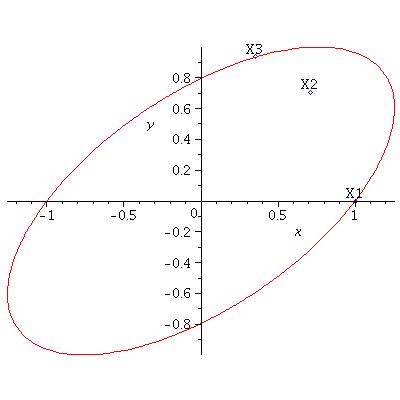}
  \subcaption{$p_1=0.5 \quad p_2=0 \quad p_3=0.5$ \label{fig:ellipse0}}
 \end{minipage} \hfill
 \begin{minipage}[b]{.43\linewidth}
  \centering
  %\subfigure[$p_1=0 \quad p_2=0.5 \quad p_3=0.5$ \label{fig:ellipse_perturbation}]{\includegraphics[width=0.8\textwidth]{ellipse2.jpg}}
  \includegraphics[width=0.8\textwidth]{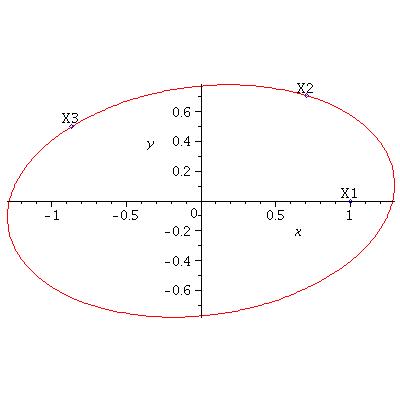}
  \subcaption{$p_1=0 \quad p_2=0.5 \quad p_3=0.5$ \label{fig:ellipse_perturbation}}
 \end{minipage}
 \vspace*{1em}
 \caption{Different minimal ellipsoids}
 \label{fig:ellipsoids}
\end{figure}

Geometrically the dual problem $(D)$ is equivalent to finding an ellipsoid containing all data points $X_k/\sigma_k$ such that the sum of the inverse of the semi-axis is maximized. The points that lie on the boundary of the ellipsoid are the one that have to be sampled. We see here that we have to sample the points that are far from the origin (after being rescaled by their standard deviation) because they cause less uncertainty.

We see that several cases can occur as shown on Figure~\ref{fig:ellipsoids}.
If one covariate is in the interior of the ellipsoid it is not sampled because of the KKT equations (see Proposition~\ref{prop:ellipsoids}). However if all the points are on the ellipsoids some of them may not be sampled. It is the case on Figure~\ref{fig:circle0} where $X_1$ is not sampled. This is due to the fact that a little perturbation of another point, for example $X_3$ can change the ellipsoid such that $X_1$ ends up inside the ellipsoid as shown on Figure~\ref{fig:ellipse_perturbation}. This case can consequently be seen as a limit case.

\end{document}